\documentclass[12pt]{article}
\pdfoutput=1

\usepackage{scicite}
\usepackage{graphicx}
\usepackage{booktabs}
\usepackage{amsmath,amsfonts,amssymb}
\usepackage[table,xcdraw]{xcolor}
\usepackage{times}
\usepackage{url}
\usepackage[normalem]{ulem}
\usepackage{comment}

\newif\ifdraft
\newif\iffinal
\newif\ifarxiv
\draftfalse
\finalfalse
\arxivtrue

\ifarxiv 
\usepackage{algorithm}
\usepackage{algpseudocode}
\usepackage{caption}
\usepackage{subcaption}
\usepackage{textcomp}
\usepackage{array}
\usepackage{calc}
\usepackage{multirow}
\usepackage{mathtools}
\newcolumntype{L}[1]{>{\raggedright\let\newline\\\arraybackslash\hspace{0pt}}m{#1}}
\newcolumntype{C}[1]{>{\centering\let\newline\\\arraybackslash\hspace{0pt}}m{#1}}
\newcolumntype{R}[1]{>{\raggedleft\let\newline\\\arraybackslash\hspace{0pt}}m{#1}}
\newtheorem{lemma}{Lemma}
\newenvironment{proof}[1][Proof]{\textbf{#1.} }{\ \rule{0.5em}{0.5em}}
\fi

\ifdraft
\usepackage[textwidth=1.5in]{todonotes}
\usepackage[top=1.6in,bottom=1.1in,left=0.5in,right=1.7in]{geometry}
\else
\usepackage[disable]{todonotes}
\fi

\newenvironment{sciabstract}{%
\begin{quote} \bf}
{\end{quote}}

\iffinal
\else
\fi

\iffalse
\newcounter{lastnote}
\newenvironment{scilastnote}{%
\setcounter{lastnote}{\value{enumiv}}%
\addtocounter{lastnote}{+1}%
\begin{list}%
{\arabic{lastnote}.}
{\setlength{\leftmargin}{.22in}}
{\setlength{\labelsep}{.5em}}}
{\end{list}}
\else
\newenvironment{scilastnote}{\paragraph{Acknowledgements.}}{}
\fi

\DeclareSymbolFont{extraup}{U}{zavm}{m}{n}
\DeclareMathSymbol{\heartsuit}{\mathalpha}{extraup}{86}
\DeclareMathSymbol{\diamondsuit}{\mathalpha}{extraup}{87}

\newcommand{\HUNL}{{HUNL}}

\newtheorem{theorem}{Theorem}
\allowdisplaybreaks 

\title{DeepStack: Expert-Level Artificial Intelligence in Heads-Up No-Limit Poker}

\author{
Matej Morav\v{c}\'{i}k$^{\spadesuit,\heartsuit,\dag}$, Martin Schmid$^{\spadesuit,\heartsuit,\dag}$, 
Neil Burch$^\spadesuit$, Viliam Lis\'y$^{\spadesuit,\clubsuit}$, \\
Dustin Morrill$^\spadesuit$, Nolan Bard$^\spadesuit$, Trevor Davis$^\spadesuit$, \\
Kevin Waugh$^\spadesuit$, Michael Johanson$^\spadesuit$, Michael Bowling$^{\spadesuit,\ast}$ \\
\normalsize{$^{\spadesuit}$Department of Computing Science, University of Alberta,}\\
\normalsize{Edmonton, Alberta, T6G2E8, Canada}\\
\normalsize{$^{\heartsuit}$Department of Applied Mathematics, Charles University,}\\
\normalsize{Prague, Czech Republic}\\
\normalsize{$^{\clubsuit}$Department of Computer Science, FEE, Czech Technical University,}\\
\normalsize{Prague, Czech Republic}\\
\\
\normalsize{$^\dag$These authors contributed equally to this work and are listed in alphabetical order.}
\\
\normalsize{$^\ast$To whom correspondence should be addressed; E-mail:  bowling@cs.ualberta.ca}
}

\date{}

\begin{document} 
\maketitle 

\begin{sciabstract}
Artificial intelligence has seen several breakthroughs in recent years, with games often serving as milestones.  A common feature of these games is that players have perfect information.  Poker is the quintessential game of imperfect information, and a longstanding challenge problem in artificial intelligence.  We introduce DeepStack, an algorithm for imperfect information settings. It combines recursive reasoning to handle information asymmetry, decomposition to focus computation on the relevant decision, and a form of intuition that is automatically learned from self-play using deep learning.  In a study involving 44,000 hands of poker, DeepStack defeated with statistical significance professional poker players in heads-up no-limit Texas hold'em.  
The approach is theoretically sound and is shown to produce more difficult to exploit strategies than prior approaches.
\iffinal\else
\footnote{This is the author's version of the work. It is posted here by permission of the AAAS for personal use, not for redistribution. The definitive version was published in {\em Science}, (March 02, 2017), doi: 10.1126/science.aam6960.}
\fi
\end{sciabstract}

Games have long served as benchmarks and marked milestones of progress in artificial intelligence (AI).  
In the last two decades, computer programs have reached a performance that exceeds expert human players in many games, e.g., backgammon~\cite{Tesauro95}, checkers~\cite{SchaefferEtAl96}, chess~\cite{CampbellEtAl02}, Jeopardy!~\cite{Ferucci12}, Atari video games~\cite{Mnih15:DQN}, and go~\cite{Silver16:AlphaGo}.
These successes all involve games with information symmetry, where all players have identical information about the current state of the game.
This property of perfect information is also at the heart of the algorithms that enabled these successes, e.g., local search during play~\cite{Samuel:AlphaBeta,Kocsis06:UCT}.

The founder of modern game theory and computing pioneer, von Neumann, envisioned reasoning in games without perfect information. ``Real life is not like that.  Real life consists of bluffing, of little tactics of deception, of asking yourself what is the other man going to think I mean to do.  And that is what games are about in my theory.''~\cite{Bronowski73}
One game that fascinated von Neumann was poker, where players are dealt private cards and take turns making bets or bluffing on holding the strongest hand, calling opponents' bets, or folding and giving up on the hand and the bets already added to the pot.
Poker is a game of imperfect information, where players' private cards give them asymmetric information about the state of game.

Heads-up no-limit Texas hold'em (HUNL) is a two-player version of poker in which two cards are initially dealt face-down to each player, and additional cards are dealt face-up in three subsequent rounds. No limit is placed on the size of the bets although there is an overall limit to the total amount wagered in each game~\cite{SOM}. AI techniques have previously shown success in the simpler game of heads-up limit Texas hold'em, where all bets are of a fixed size resulting in just under $10^{14}$ decision points~\cite{progress-in-hulhe}.
By comparison, computers have exceeded expert human performance in go~\cite{Silver16:AlphaGo}, a perfect information game with approximately $10^{170}$ decision points~\cite{allis1994}. The imperfect information game HUNL is comparable in size to go, with the number of decision points exceeding $10^{160}$~\cite{Johanson13:Sizes}.

Imperfect information games require more complex reasoning than similarly sized perfect information games. 
The correct decision at a particular moment depends upon the probability distribution over private information that the opponent holds, which is revealed through their past actions.
However, how our opponent's actions reveal that information depends upon their knowledge of our private information and how our actions reveal it.
This kind of recursive reasoning is why one cannot easily reason about game situations in isolation, which is at the heart of heuristic search methods for perfect information games.  
Competitive AI approaches in imperfect information games typically reason about the entire game and produce a complete strategy prior to play~\cite{ZinkevichEtAl07,GilpinEtAl07,endgame-solving-exception}.
Counterfactual regret minimization (CFR)~\cite{ZinkevichEtAl07,cprg:cfrd,Bowling15:Cepheus} is one such technique that uses self-play to do recursive reasoning through adapting its strategy against itself over successive iterations.  If the game is too large to be solved directly, the common response is to solve a smaller, abstracted game.  To play the original game, one translates situations and actions from the original game to the abstract game.

Although this approach makes it feasible for programs to reason in a game like \HUNL, it does so by squeezing \HUNL's $10^{160}$ situations down to the order of $10^{14}$ abstract situations.  Likely as a result of this loss of information, such programs are behind expert human play.  In 2015, the computer program Claudico lost to a team of professional poker players by a margin of 91 mbb/g~\cite{milli-big-blinds},
which is a ``huge margin of victory''~\cite{Wood15:PokerFuse-Claudico}.  Furthermore, it has been recently shown that abstraction-based programs from the Annual Computer Poker Competition have massive flaws~\cite{Lisy17:LocalBR}. Four such programs (including top programs from the 2016 competition) were evaluated using a local best-response technique that produces an approximate lower-bound on how much a strategy can lose.  All four abstraction-based programs are beatable by over 3,000 mbb/g, which is four times as large as simply folding each game.

DeepStack takes a fundamentally different approach.  It continues to use the recursive reasoning of CFR to handle information asymmetry.  However, it does not compute and store a complete strategy prior to play and so has no need for explicit abstraction.  Instead it considers each particular situation as it arises during play, but not in isolation.  It avoids reasoning about the entire remainder of the game by substituting the computation beyond a certain depth with a fast approximate estimate.  This estimate can be thought of as DeepStack's intuition: a gut feeling of the value of holding any possible private cards in any possible poker situation.  Finally, DeepStack's intuition, much like human intuition, needs to be trained.  We train it with deep learning~\cite{deep-learning-in-poker} using examples generated from random poker situations.  We show that DeepStack is theoretically sound, produces strategies substantially more difficult to exploit than abstraction-based techniques, and defeats professional poker players at \HUNL{} with statistical significance. 

\section*{DeepStack}

DeepStack is a general-purpose algorithm for a large class of sequential imperfect information games.  For clarity, we will describe its operation in the game of \HUNL.  The state of a poker game
can be split into the players' private information, hands of two cards dealt face down,
and the public state,
consisting of the cards laying face up on the table and the sequence
of betting actions made by the players. Possible sequences of public states in
the game form a public tree with every public state having an associated public subtree (Fig.~\ref{fig:public_tree}).

\iffinal\else 
\begin{figure}[t]
\centering
\includegraphics[width=0.8\textwidth]{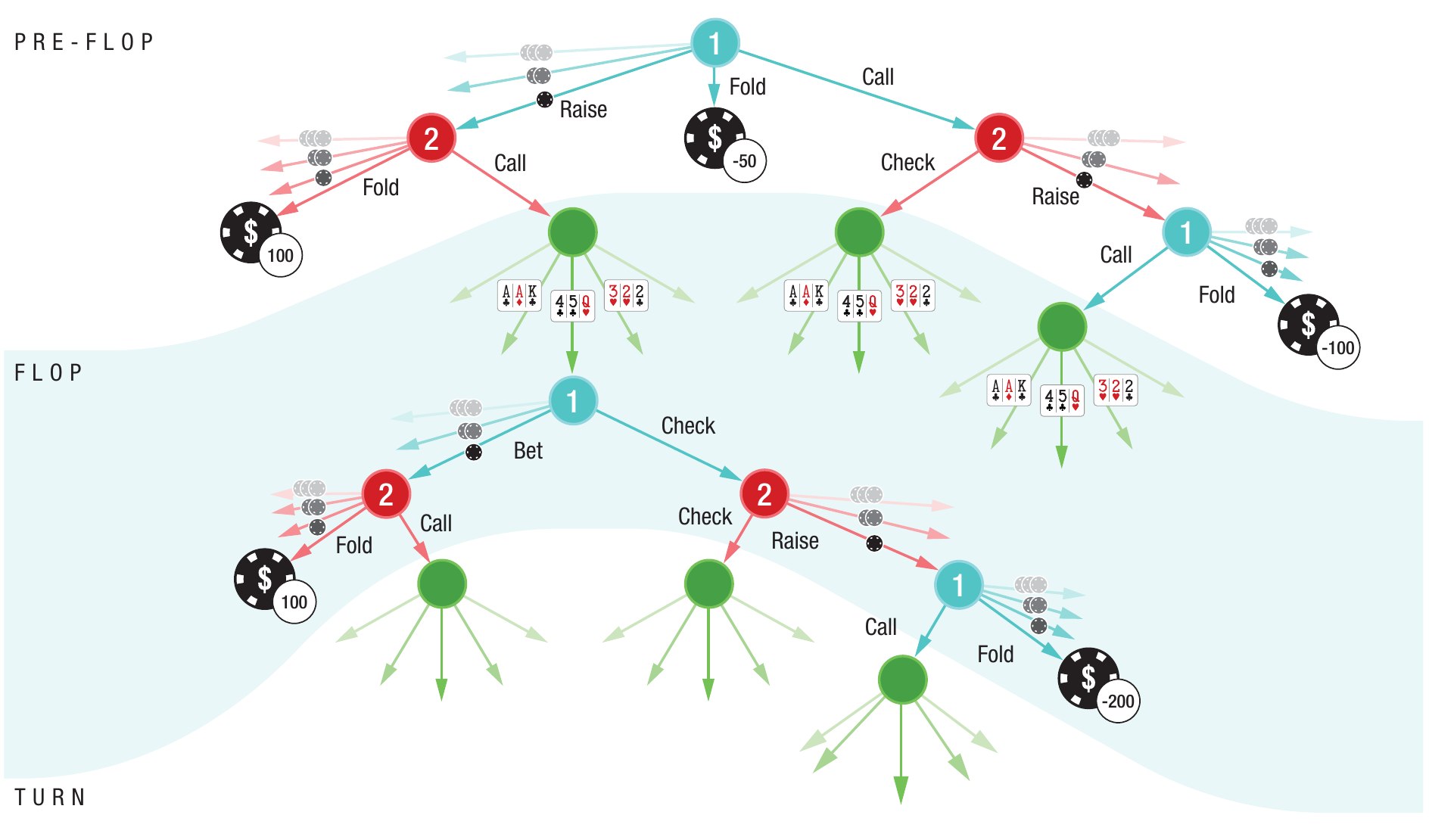}
\caption{A portion of the public tree in \HUNL{}.  Nodes represent public states, whereas edges represent actions: red and turquoise showing player betting actions, and green representing public cards revealed by chance.  The game ends at terminal nodes, shown as a chip with an associated value.  For terminal nodes where no player folded, the player whose private cards form a stronger poker hand receives the value of the state.}\label{fig:public_tree}
\end{figure}
\fi

A player's strategy defines a probability distribution over
valid actions for each decision point, where a decision point is the combination
of the public state and the hand for the acting player. Given a player's
strategy, for any public state one can compute the player's
range, which is the probability distribution over the player's
possible hands given that the public state is reached.

Fixing both players' strategies, the utility for a particular player
at a terminal public state, where the
game has ended, is a bilinear function of both players' ranges
using a payoff matrix determined by the rules of the game. The
expected utility for a player at any other public state, including the initial
state, 
is the expected utility over reachable terminal states 
given the players'
fixed strategies. A best-response strategy is one that
maximizes a player's expected utility against an opponent strategy.
In two-player zero-sum games, like \HUNL, a solution or
Nash equilibrium strategy~\cite{Nash50:equilibrium} maximizes
the expected
utility when playing against a best-response opponent strategy.
The exploitability of a strategy is the difference in expected
utility against its best-response opponent and the 
expected utility under a Nash equilibrium.  

The DeepStack algorithm seeks to compute and play a low-exploitability strategy for the game, i.e., solve for an approximate Nash equilibrium.  DeepStack computes this strategy during play only for the states of the public tree that actually arise.  Although computed during play, DeepStack's strategy is static, albeit stochastic, because it is the result of a deterministic computation that produces a probability distribution over the available actions.

The DeepStack algorithm (Fig.~\ref{fig:deepstack-overview}) is composed of three ingredients: a sound local strategy computation for the current public state, depth-limited lookahead using a learned value function to avoid reasoning to the end of the game, and a restricted set of lookahead actions.  
At a conceptual level these three ingredients describe heuristic search, which is responsible for many of AI's successes in perfect information games.
Until DeepStack, no theoretically sound application of heuristic search was known in imperfect information games.
The heart of heuristic search methods is the idea of ``continual re-searching'', where a sound local search procedure is invoked whenever the agent must act without retaining any memory of how or why it acted to reach the current state.  At the heart of DeepStack is continual re-solving, a sound local strategy computation which only needs minimal memory of how and why it acted to reach the current public state.  
 
\paragraph*{Continual re-solving.}
Suppose we have taken actions according to a particular solution strategy 
but then in  some public state forget this strategy.  Can we reconstruct 
a solution strategy for the subtree without having to solve the entire game again?
We can, through the process of re-solving~\cite{cprg:cfrd}.  We need to know both our range at the public state and a vector of expected values achieved by the opponent under the previous solution for each opponent hand~\cite{re-solving-values}.  With these values, we can reconstruct a strategy for only the remainder of the game, which does not increase our overall exploitability.
Each value in the opponent's vector is a counterfactual value, a conditional ``what-if'' value that gives the expected
value if the opponent reaches the public state with a particular hand.  The CFR algorithm also uses
counterfactual values, and if we use CFR as our solver, it is easy to compute the vector of opponent
counterfactual values at any public state.  

Re-solving, however, begins with a strategy, whereas our goal is to avoid ever maintaining a strategy for the entire game.  We get around this by doing continual re-solving: reconstructing a strategy by re-solving every time we need to act; never using the strategy beyond our next action.  To be able to re-solve at any public state, we need only keep track of our own range and a suitable vector of opponent counterfactual values.  These values must be an upper bound on the value the opponent can achieve with each hand in the current public state, while being no larger than the value the opponent could achieve had they deviated from reaching the public state.
This is an important relaxation of the counterfactual values typically used in re-solving, with a proof of sufficiency included in our proof of Theorem~\ref{thm} below~\cite{SOM}.

At the start of the game, our range is uniform and the opponent counterfactual values are initialized to the value of being dealt each private hand.  When it is our turn to act we re-solve the subtree at the current public state using the stored range and opponent values, and act according to the computed strategy, discarding the strategy before we act again.  After each action, either by a player or chance dealing cards, we update our range and opponent counterfactual values according to the following rules: (i) Own action: replace the opponent 
counterfactual values with those computed in the re-solved strategy for our chosen action.
Update our own range using the computed strategy and Bayes' rule.  (ii) Chance action: replace the opponent counterfactual values with those computed for this chance action from the last re-solve.  Update our own range by zeroing hands in the range that are impossible given new public cards. (iii) Opponent action: no change to our range or the opponent values are required.

These updates ensure the opponent counterfactual values satisfy our sufficient conditions, and the whole procedure produces arbitrarily close approximations of a Nash equilibrium (see Theorem~\ref{thm}).
Notice that continual re-solving never keeps track of the opponent's range, instead only keeping track of their counterfactual values.  Furthermore, it never requires knowledge of the opponent's action to update these values, which is an important difference from traditional re-solving.  Both will prove key to making this algorithm efficient and avoiding any need for the translation step required with action abstraction methods~\cite{Gilpin08:Tartanian,Schnizlein09:Translation}.

Continual re-solving is theoretically sound, but by itself impractical.  
While it does not ever maintain a complete strategy, re-solving itself is
intractable except near the end of the game. 
In order to make continual re-solving practical, we need to limit the depth and breadth of the re-solved subtree.

\paragraph*{Limited depth lookahead via intuition.} 
As in heuristic search for perfect information games, we would like to limit the depth of the subtree we have to reason about when re-solving.  However, in imperfect information games we cannot simply replace a subtree with a heuristic or precomputed value.  The counterfactual values at a public state are not fixed, but depend on how players play to reach the public state, i.e., the players' ranges~\cite{cprg:cfrd}.  When using an iterative algorithm, such as CFR, to re-solve, these ranges change on each iteration of the solver.  

DeepStack overcomes this challenge by replacing subtrees beyond a certain depth with a learned counterfactual value function that approximates the resulting values if that public state were to be solved with the current iteration's ranges.  The inputs to this function are the ranges for both players, as well as the pot size and public cards, which are sufficient to specify the public state.  The outputs are a vector for each player containing the counterfactual values of holding each hand in that situation.  In other words, the input is itself a description of a poker game: the probability distribution of being dealt individual private hands, the stakes of the game, and any public cards revealed; the output is an estimate of how valuable holding certain cards would be in such a game.  The value function is a sort of intuition, a fast estimate of the value of finding oneself in an arbitrary poker situation.  With a depth limit of four actions, this approach reduces the size of the game for re-solving from $10^{160}$ decision points at the start of the game down to no more than $10^{17}$ decision points.  DeepStack uses a deep neural network as its learned value function, which we describe later.

\paragraph*{Sound reasoning.}
DeepStack's depth-limited continual re-solving is sound.  If DeepStack's intuition is ``good'' and ``enough'' computation is used in each re-solving step, then DeepStack plays an arbitrarily close approximation to a Nash equilibrium.
\begin{theorem}
If the values returned by the value function used when the depth limit is reached have error less than $\epsilon$, and $T$ iterations of CFR are used to re-solve, then the resulting strategy's exploitability is less than $k_1\epsilon + k_2 / \sqrt{T}$, where $k_1$ and $k_2$ are game-specific constants.  For the proof, see \cite{SOM}.
\label{thm}
\end{theorem}

\paragraph*{Sparse lookahead trees.} 
The final ingredient in DeepStack is the reduction in the number of actions considered so as to construct a sparse lookahead tree.
DeepStack builds the lookahead tree using only the actions fold (if valid), 
call, 2 or 3 bet actions, and all-in.  
This step voids the soundness property of Theorem~\ref{thm}, but it allows DeepStack to play at conventional human speeds.
With sparse and depth-limited
lookahead trees, the re-solved games have approximately $10^{7}$ decision points, and are solved 
in under five seconds using a single NVIDIA GeForce GTX 1080 graphics card.  We also use the sparse and depth-limited lookahead solver from the start of the game to compute the opponent counterfactual values used to initialize DeepStack's continual re-solving.

\iffinal\else 
\begin{figure}[t]
\centering
\includegraphics[width=0.7\textwidth]{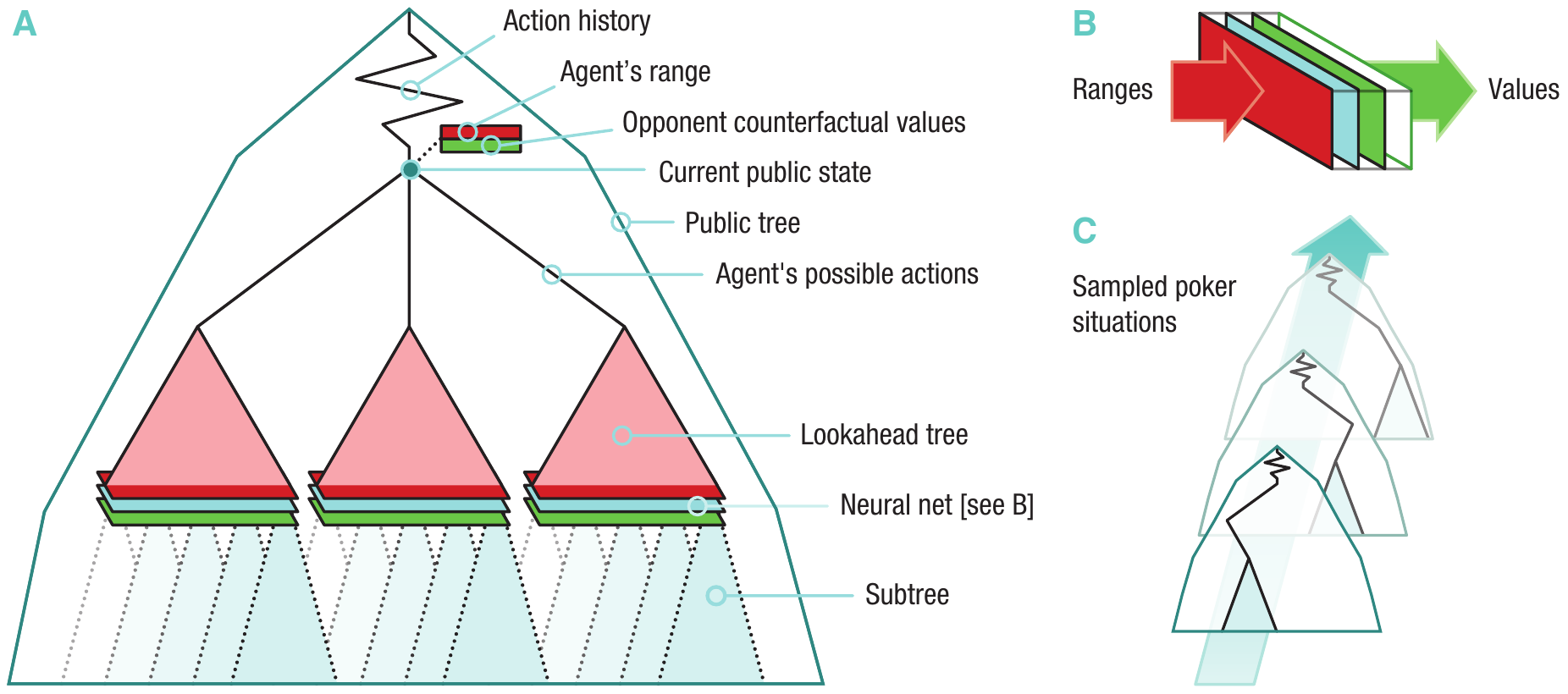}
\caption{%
{\bf DeepStack overview.} 
{\bf (A)} DeepStack reasons in the public tree always producing action probabilities for all cards it can hold in a public state.  It maintains two vectors while it plays: its own range and its opponent's counterfactual values.  As the game proceeds, its own range is updated via Bayes' rule using its computed action probabilities after it takes an action.  Opponent counterfactual values are updated as discussed under ``Continual re-solving''.  To compute action probabilities when it must act, it performs a re-solve using its range and the opponent counterfactual values.  To make the re-solve tractable it restricts the available actions of the players and lookahead is limited to the end of the round.
During the re-solve, counterfactual values for public states beyond its lookahead are approximated using DeepStack's learned evaluation function.  {\bf (B)} The evaluation function is represented with a neural network that takes the public state and ranges from the current iteration as input and outputs counterfactual values for both players (Fig.~\ref{fig:dnn}).  {\bf (C)} The neural network is trained prior to play by generating random poker situations (pot size, board cards, and ranges) and solving them to produce training examples.  Complete pseudocode can be found in Algorithm S1~\cite{SOM}.%
}
\label{fig:deepstack-overview}
\end{figure}
\fi

\paragraph*{Relationship to heuristic search in perfect information games.}

There are three key challenges that DeepStack overcomes to incorporate heuristic search ideas in imperfect information games.
First, sound re-solving of public states cannot be done without knowledge of how and why the players acted to reach the public state.  Instead, two additional vectors, the agent's range and opponent counterfactual values, must be maintained to be used in re-solving.  Second, re-solving is an iterative process that traverses the lookahead tree multiple times instead of just once.  Each iteration requires querying the evaluation function again with different ranges for every public state beyond the depth limit.
Third, the evaluation function needed when the depth limit is reached is conceptually more complicated than in the perfect information setting.  Rather than returning a single value given a single state in the game, the counterfactual value function needs to return a vector of values given the public state and the players' ranges.  
Because of this complexity, to learn such a value function we use deep learning, which has also been successful at learning complex evaluation functions in perfect information games~\cite{Silver16:AlphaGo}.

\paragraph*{Relationship to abstraction-based approaches.}

Although DeepStack uses ideas from abstraction, it is fundamentally different from abstraction-based approaches.  DeepStack restricts the number of actions in its lookahead trees, much like action abstraction~\cite{Gilpin08:Tartanian,Schnizlein09:Translation}.  However, each re-solve in DeepStack starts from the actual public state and 
 so it always perfectly understands the current situation.  The algorithm also never needs to use the opponent's actual action to obtain correct ranges or opponent counterfactual values, 
thereby avoiding translation of opponent bets.  We used hand clustering as inputs to our counterfactual value functions, much like explicit card abstraction approaches~\cite{Gilpin07:Abstraction,Johanson13:Abstraction}.  However, our clustering is used to estimate counterfactual values at the end of a lookahead tree rather than limiting what information the player has about their cards when acting.  
We later show that these differences result in a strategy substantially more difficult to exploit.

\section*{Deep Counterfactual Value Networks}

Deep neural networks have proven to be powerful models and are responsible for major advances in image and speech recognition\cite{Krizhevsky12:ImageNet,Hinton12:Speech}, automated generation of music \cite{Oord16:WaveNet}, and game-playing~\cite{Mnih15:DQN,Silver16:AlphaGo}.
DeepStack uses deep neural networks with a tailor-made architecture, as the value function for its depth-limited lookahead (Fig. \ref{fig:dnn}).  Two separate networks are trained: one estimates the counterfactual values after the first three public cards are dealt (flop network), the other after dealing the fourth public card (turn network). An auxiliary network for values before any public cards are dealt is used to speed up the re-solving for early actions~\cite{SOM}.

\paragraph*{Architecture.}  DeepStack uses a standard feedforward network with seven fully connected hidden layers each with 500 nodes and parametric rectified linear units~\cite{He15:PReLU} for the output.  This architecture is embedded in an outer network that forces the counterfactual values to satisfy the zero-sum property.  The outer computation takes the estimated counterfactual values, and computes a weighted sum using the two players' input ranges resulting in separate estimates of the game value.  These two values should sum to zero, but may not.  Half the actual sum is then subtracted from the two players' estimated counterfactual values.  
This entire computation is differentiable and can be trained with gradient descent.  The network's inputs are the pot size as a fraction of the players' total stacks and an encoding of the players' ranges as a function of the public cards.  The ranges are encoded by clustering hands into 1,000 buckets, as in traditional abstraction methods~\cite{Shi00:Rhode,Gilpin07:Abstraction,Johanson13:Abstraction}, and input as a vector of probabilities over the buckets.  
The output of the network are vectors of counterfactual values for each player and hand, interpreted as fractions of the pot size.

\iffinal\else 
\begin{figure}[t]
\centering
\includegraphics[width=0.8\textwidth]{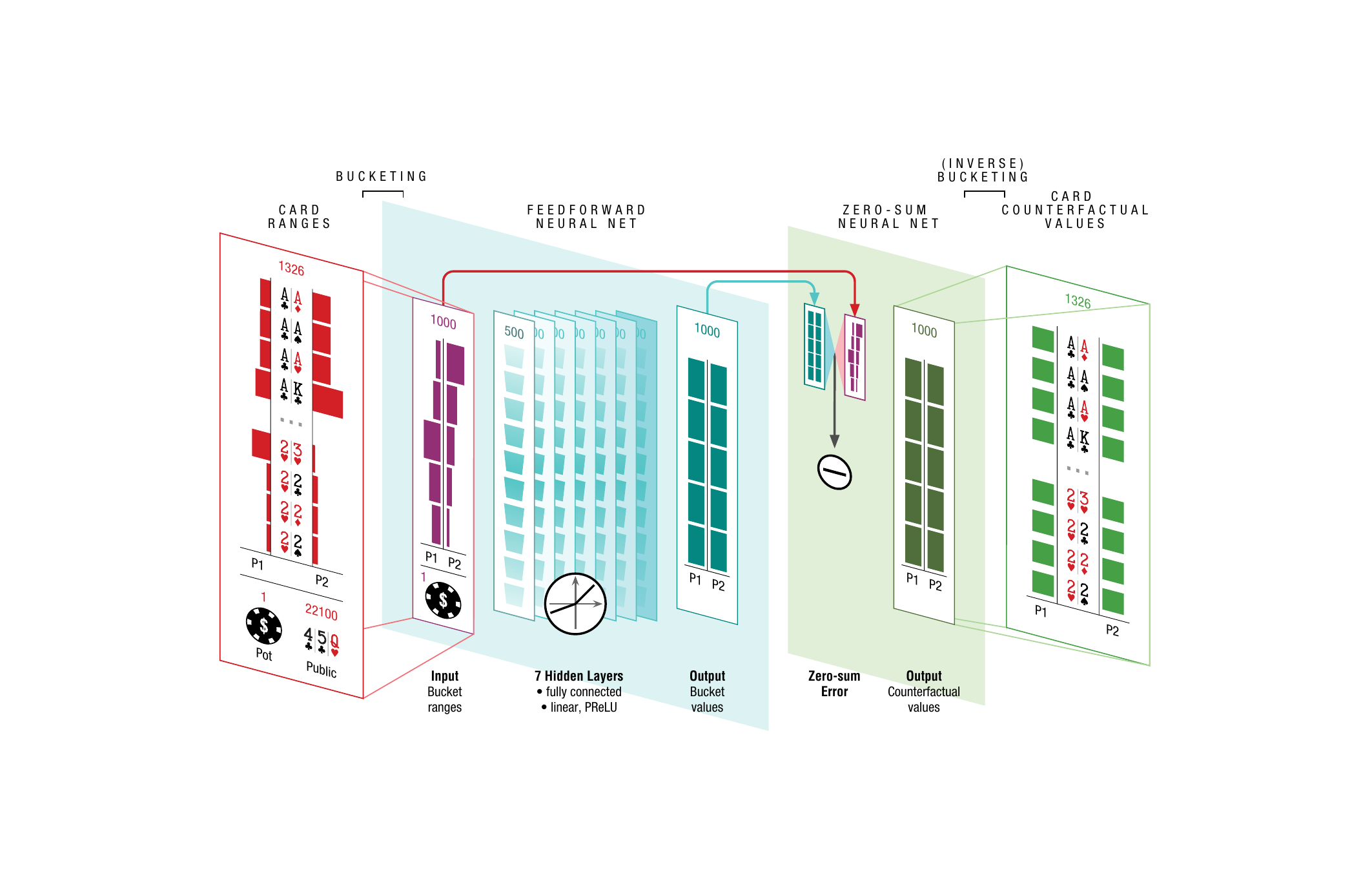}
\caption{{\bf Deep counterfactual value network.}  The inputs to the network are the pot size, public cards, and the player ranges, which are first processed into hand clusters.  The output from the seven fully connected hidden layers is post-processed to guarantee the values satisfy the zero-sum constraint, and then mapped back into a vector of counterfactual values. }\label{fig:dnn}
\end{figure}
\fi

\paragraph*{Training.}  The turn network was trained by solving 10 million randomly generated poker turn games.  These turn games used randomly generated ranges, public cards, and a random pot size~\cite{SOM}.  The target counterfactual values for each training game were generated by solving the game with players' actions restricted to fold, call, a pot-sized bet, and an all-in bet, but no card abstraction.  The flop network was trained similarly with 1 million randomly generated flop games.  However, the target counterfactual values were computed using our depth-limited solving procedure and our trained turn network.  The networks were trained using the Adam gradient descent optimization procedure\cite{kingma2014adam} with a Huber loss \cite{huber1964}. 

\section*{Evaluating DeepStack}

We evaluated DeepStack by playing it against a pool of professional poker players recruited by the International Federation of Poker\cite{IFP}.  Thirty-three players from 17 countries were recruited.  Each was asked to complete a 3,000 game match over a period of four weeks between November 7th and December 12th, 2016.  Cash incentives were given to the top three performers (\$5,000, \$2,500, and \$1,250 CAD).  

Evaluating performance in \HUNL{} is challenging because of the large variance in per-game outcomes owing to randomly dealt cards and stochastic choices made by the players.  The better player may lose in a short match simply because they were dealt weaker hands or their rare bluffs were made at inopportune times.  As seen in the Claudico match~\cite{Wood15:PokerFuse-Claudico}, even 80,000 games may not be enough to statistically significantly separate players whose skill differs by a considerable margin.  We evaluate performance using AIVAT~\cite{Burch17:AIVAT}, a provably unbiased low-variance technique for evaluating performance in imperfect information games based on carefully constructed control variates.  AIVAT requires an estimated value of holding each hand in each public state, and then uses the expected value changes that occur due to chance actions and actions of players with known strategies (i.e., DeepStack) to compute the control variate.  DeepStack's own value function estimate is perfectly suited for AIVAT.  Indeed, when used with AIVAT we get an unbiased performance estimate with an impressive $85\%$ reduction in standard deviation.  Thanks to this technique, we can show statistical significance~\cite{stat-sig} in matches with as few as 3,000 games.  

\iffinal\else 
\begin{figure}[t]
\centering
\includegraphics[width=0.7\textwidth]{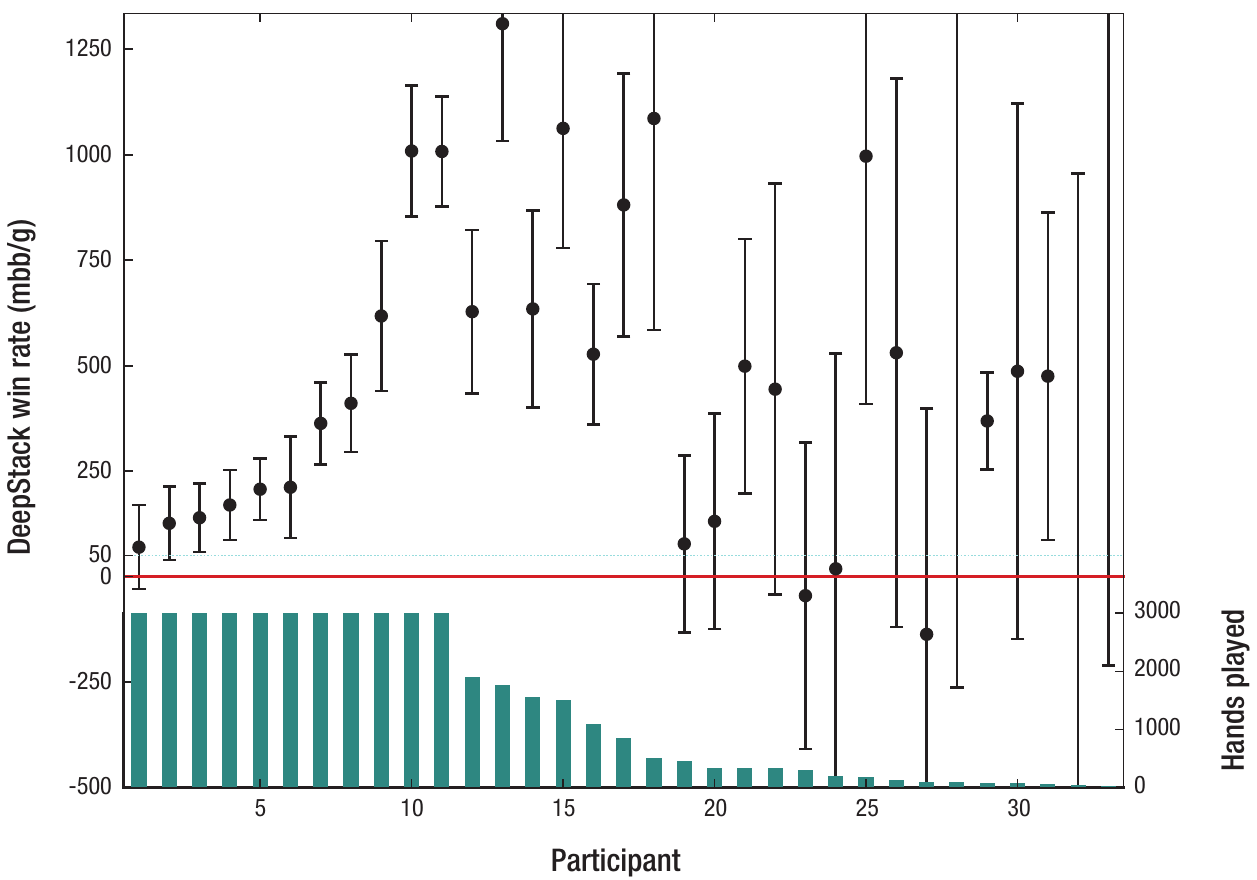}
\caption{{\bf Performance of professional poker players against DeepStack.}  Performance estimated with AIVAT along with a 95\% confidence interval. The solid bars at the bottom show the number of games the participant completed.}\label{fig:humans}
\end{figure}
\fi

In total 44,852 games were played by the thirty-three players with 11 players completing the requested 3,000 games.
Over all games played, DeepStack won 492 mbb/g.  This is over 4 standard deviations away from zero, and so, highly significant.  Note that professional poker players consider 50 mbb/g a sizable margin.  Using AIVAT to evaluate performance, we see DeepStack was overall a bit lucky, with its estimated performance actually 486 mbb/g.  However, as a lower variance estimate, this margin is over 20 standard deviations from zero.

The performance of individual participants measured with AIVAT is summarized in Figure~\ref{fig:humans}.
Amongst those players that completed the requested 3,000 games, DeepStack is estimated to be winning by 394 mbb/g, and individually beating 10 out of 11 such players by a statistically significant margin.  Only for the best performing player, still estimated to be losing by 70 mbb/g, is the result not statistically significant.
More details on the participants and their results are presented in~\cite{SOM}.

\subsection*{Exploitability}

The main goal of DeepStack is to approximate Nash equilibrium play, i.e., minimize exploitability.
While the exact exploitability of a \HUNL{} poker strategy is intractable to compute, the recent local best-response technique (LBR) can provide a lower bound on a strategy's exploitability~\cite{Lisy17:LocalBR} given full access to its action probabilities.  LBR uses the action probabilities to compute the strategy's range at any public state.  Using this range it chooses its response action from a fixed set using the  assumption that no more bets will be placed for the remainder of the game.  Thus it best-responds locally to the opponent's actions, providing a lower bound on their overall exploitability.  As already noted, abstraction-based programs from the Annual Computer Poker Competition are highly exploitable by LBR: four times more exploitable than folding every game (Table~\ref{tab-localbr-abbrev}).  However, even under a variety of settings, LBR fails to exploit DeepStack at all --- itself losing by over 350 mbb/g to DeepStack~\cite{SOM}.  Either a more sophisticated lookahead is required to identify DeepStack's weaknesses or it is substantially less exploitable. 

\iffinal\else
\begin{table}
\centering
\caption{%
{\bf Exploitability bounds from Local Best Response.}  For all listed programs, the value reported is the largest estimated exploitability when applying LBR using a variety of different action sets.  Table S2 gives a more complete presentation of these results~\cite{SOM}.  $\ddagger$: LBR was unable to identify a positive lower bound for DeepStack's exploitability.%
}
\label{tab-localbr-abbrev}
\begin{tabular}{c|r}
\toprule
Program & LBR (mbb/g) \\
\midrule
Hyperborean (2014) & 4675 \\
Slumbot (2016) & 4020 \\
Act1 (2016) & 3302 \\
Always Fold & 750 \\
\midrule
{\bf DeepStack} & 
{\bf 0 $\ddagger$} \\
\bottomrule
\end{tabular}
\end{table}
\fi

\section*{Discussion}

DeepStack defeated professional poker players at \HUNL{} with statistical significance~\cite{Libratus}, a game that is similarly sized to go, but with the added complexity of imperfect information.
It achieves this goal with little domain knowledge and no training from expert human games.  The implications go beyond being a milestone for artificial intelligence.  DeepStack represents a paradigm shift in approximating solutions to large, sequential imperfect information games.  
Abstraction and offline computation of complete strategies has been the dominant approach for almost 20 years~\cite{Shi00:Rhode,BillingsEtAl03,Sandholm12}.  
DeepStack allows computation to be focused on specific situations that arise when making decisions and the use of automatically trained value functions.  These are two of the core principles that have powered successes in perfect information games, albeit conceptually simpler to implement in those settings.  As a result, the gap between the largest perfect and imperfect information games to have been mastered is mostly closed.

With many real world problems involving information asymmetry, DeepStack also has implications for seeing powerful AI applied more in settings that do not fit the perfect information assumption.  The abstraction paradigm for handling imperfect information has shown promise in applications like defending strategic resources~\cite{Lisy16:cfr-security} and robust decision making as needed for medical treatment recommendations~\cite{ChenBowling12}.  DeepStack's continual re-solving paradigm will hopefully open up many more possibilities.  

\bibliography{paper}

\begin{thebibliography}{10}
\bibitem{endgame-solving-exception}
End-game solving~\cite{Ganzfried15:Endgame,cprg:cfrd,Moravcik16:Subgames} is
  one exception to computation occurring prior to play. When the game nears the
  end, a new computation is invoked over the remainder of the game. Thus, the
  program need not store this part of the strategy or can use a finer-grained
  abstraction aimed to improve the solution quality. We discuss this as
  re-solving when we introduce DeepStack's technique of continual re-solving.
\bibitem{milli-big-blinds}
We use milli-big-blinds per game (mbb/g) to measure performance in poker, where
  a milli-big-blind is one thousandth of the forced big blind bet amount that
  starts the game. This normalizes performance for the number of games played
  and the size of stakes. For comparison, a win rate of 50 mbb/g is considered
  a sizable margin by professional players and 750 mbb/g is the rate that would
  be lost if a player folded each game. The poker community commonly uses big
  blinds per one hundred games (bb/100) to measure win rates, where 10 mbb/g
  equals 1 bb/100.
\bibitem{re-solving-values}
When the previous solution is an approximation, rather than an exact Nash
  equilibrium, sound re-solving needs the expected values from a best-response
  to the player's previous strategy. In practice, though, using expected values
  from the previous solution in self-play often works as well or better than
  best-response values~\cite{SOM}.
\bibitem{stat-sig}
Statistical significance where noted was established using a two-tailed
  Student's $t$-test at the $0.05$ level with $N \ge 3000$.
\end{thebibliography}
\bibliographystyle{Science}

\nocite{ZinkevichLittman06}
\nocite{DustinGUISource}
\nocite{Tammelin15:CFR+}
\nocite{collobert2011torch7}
\nocite{Ganzfried14:EMD}

\begin{scilastnote}
The hand histories of all games played in the human study as well as those used to generate the LBR results against DeepStack are in~\cite{SOM}.
We would like to thank the IFP, all of the professional players who committed valuable time to play against DeepStack, the anonymous reviewers' invaluable feedback and suggestions, and
R.~Holte, A.~Brown, and K.~Bla\v{z}kov\'a for comments on early drafts of this article.
We especially would like to thank IBM for their support of this research through an IBM faculty grant.  
The research was also supported by Alberta Innovates through the Alberta Machine Intelligence Institute, the Natural Sciences and Engineering Research Council of Canada, and Charles University (GAUK) Grant no. 391715.  
This work was only possible thanks to computing resources provided by Compute Canada and Calcul Qu\'ebec.  
MM and MS are on leave from IBM Prague. MJ serves as a Research Scientist, and MB as a Contributing Brain Trust Member, at Cogitai, Inc. DM owns 200 shares of Gamehost Inc.
\end{scilastnote}

\iffinal
\subsubsection*{Supplementary Materials}

\begin{trivlist}
\item Supplementary Text
\item Figs. S1 and S2, Tables S1 -- S6, Algorithm S1
\item Hand histories of all games played in the human study with AIVAT analysis
\item Hand histories of all games played between LBR and DeepStack (plus LBR and other University of Alberta poker programs)
\end{trivlist}

\subsubsection*{Figure Captions}
\noindent 
Figure 1: 
{\bf A portion of the public tree in \HUNL{}.}  Nodes represent public states, whereas edges represent actions: red and turquoise showing player betting actions, and green representing public cards revealed by chance.  The game ends at terminal nodes, shown as a chip with an associated value.  For terminal nodes where no player folded the value is returned by a function of the players' joint private information.
\bigskip

\noindent 
Figure 2: 
{\bf DeepStack overview.} 
{\bf (A)} DeepStack reasons in the public tree always producing action probabilities for all cards it can hold in a public state.  It maintains two vectors while it plays: its own range and its opponent's counterfactual values.  As the game proceeds, its own range is updated via Bayes' rule using its computed action probabilities after it takes an action.  Opponent counterfactual values are updated as discussed under ``Continual re-solving''.  To compute action probabilities when it must act, it performs a re-solve using its range and the opponent counterfactual values.  To make the re-solve tractable it restricts the available actions of the players and lookahead is limited to the end of the round.
During the re-solve, counterfactual values for public states beyond its lookahead are approximated using DeepStack's learned evaluation function.  {\bf (B)} The evaluation function is represented with a neural network that takes the public state and ranges from the current iteration as input and outputs counterfactual values for both players (Fig.~\ref{fig:dnn}).  {\bf (C)} The neural network is trained prior to play by generating random poker situations (pot size, board cards, and ranges) and solving them to produce training examples.  Complete pseudocode can be found in Algorithm S1~\cite{SOM}.%
\bigskip

\noindent 
Figure 3:
{\bf Deep counterfactual value network.}  The inputs to the network are the pot size, public cards, and the player ranges, which are first processed into hand clusters.  The output from the seven fully connected hidden layers is post-processed to guarantee the values satisfy the zero-sum constraint, and then mapped back into a vector of counterfactual values.
\bigskip

\noindent 
Figure 4:
{\bf Performance of professional poker players against DeepStack.}  Performance estimated with AIVAT along with a 95\% confidence interval. The solid bars at the bottom show the number of games the participant completed.
\bigskip

\noindent
Table 1:
{\bf Exploitability bounds from Local Best Response.}  For all listed programs, the value reported is the largest estimated exploitability when applying LBR using a variety of different action sets.  Table S2 gives a more complete presentation of these results~\cite{SOM}.  $\ddagger$: LBR was unable to identify a positive lower bound for DeepStack's exploitability.%
\fi 

\ifarxiv
\clearpage
\begin{center}
{\LARGE\baselineskip24pt 
Supplementary Materials for \\
{\large DeepStack: Expert-Level AI in No-Limit Poker}
\par\bigskip\bigskip\bigskip
}
\end{center}

\renewcommand{\topfraction}{.85}
\renewcommand{\bottomfraction}{.7}
\renewcommand{\textfraction}{.15}
\renewcommand{\floatpagefraction}{.66}
\renewcommand{\dbltopfraction}{.66}
\renewcommand{\dblfloatpagefraction}{.66}
\setcounter{topnumber}{9}
\setcounter{bottomnumber}{9}
\setcounter{totalnumber}{20}
\setcounter{dbltopnumber}{9}

\section*{Game of Heads-Up No-Limit Texas Hold'em}
Heads-up no-limit Texas hold'em (\HUNL{}) is a two-player poker game. It is a repeated game, in which the
two players play a match of individual games, usually called hands,
while alternating
who is the dealer. In each of the individual games, one player will win some number of chips
from the other player, and the goal is to win as many chips as possible over the course of the
match.  

Each individual game begins with both players placing a number of chips in the pot: the
player in the dealer position puts in the small blind, and the other player puts in the big blind,
which is twice the small blind amount.   
During a game, a player can only wager and win up to a fixed amount known as their stack.  In the
particular format of \HUNL{} used in the Annual Computer Poker Competition~\cite{ZinkevichLittman06} and this article, the big blind is 100 chips and the stack is 20,000 chips or 200 big blinds.
Resetting the stacks after each game is called ``Doyle's Game'', named for the
professional poker player Doyle Brunson who
publicized this variant~\cite{Gilpin08:Tartanian}.  It is used in the Annual Computer Poker Competitions because it allows for each game to be an independent sample of the same game.

A game of \HUNL{} progresses through four rounds: the
pre-flop, flop, turn, and river. Each round consists of cards being dealt followed by player
actions in the form of wagers as to who will hold the strongest hand at the end of the game. In
the pre-flop, each player is given two private cards, unobserved by their opponent. In the later
rounds, cards are dealt face-up in the center of the table, called public cards. A total of five
public cards are revealed over the four rounds: three on the flop, one on the turn, and one
on the river.

After the cards for the round are dealt, players alternate taking actions of
three types: fold, call, or raise. A player folds by declining to match the last opponent wager,
thus forfeiting to the opponent all chips in the pot and ending the game with no player revealing
their private cards. A player calls by adding chips into the pot to match the last opponent wager,
which causes the next round to begin. A player raises by adding chips into the pot to match the
last wager followed by adding additional chips to make a wager of their own. At the beginning
of a round when there is no opponent wager yet to match, the raise action is called bet, and
the call action is called check, which only ends the round if both players check. 
An all-in wager is one
involving all of the chips remaining the player's stack.  If the wager is called, there is no further 
wagering in later rounds.
The size of
any other wager can be any whole number of chips remaining in the player's stack, as long as it
is not smaller than the last wager in the current round or the big blind.  

The dealer acts first in the pre-flop round and must decide whether to fold, call, or raise the opponent's big blind bet.  In all subsequent rounds, the non-dealer acts first.
If the river round ends with no player previously folding to end the game, the outcome is
determined by a showdown. Each player reveals their two private cards and the player that can
form the strongest five-card poker hand (see ``List of poker hand categories'' on Wikipedia; accessed January
1, 2017) wins all the chips in the pot. To form their hand each player may use any cards from
their two private cards and the five public cards. At the end of the game, whether ended by
fold or showdown, the players will swap who is the dealer and begin the next game.

Since the game can be played for different stakes, such as a big blind being worth \$0.01
or \$1 or \$1000, players commonly measure their performance over a match as their average
number of big blinds won per game. Researchers have standardized on the unit milli-big-blinds
per game, or mbb/g, where one milli-big-blind is one thousandth of one big blind. A player that always
folds will lose 750 mbb/g (by losing 1000 mbb as the big blind and 500 as the small blind).
A human rule-of-thumb is that a professional should aim to win at least 50 mbb/g from their
opponents. Milli-big-blinds per game is also used as a unit of exploitability, when it is computed
as the expected loss per game against a worst-case opponent.  In the poker community, it is common to use big blinds per one hundred games (bb/100) to measure win rates, where 10 mbb/g equals 1 bb/100.

\section*{Poker Glossary}
\begin{description}
\item[all-in] A wager of the remainder of a player's stack.  The opponent's only response can be call or fold.
\item[bet] The first wager in a round; putting more chips into the pot. 
\item[big blind] Initial wager made by the non-dealer before any cards are dealt.  The big blind is twice the size of the small blind.  
\item[call] Putting enough chips into the pot to match the current wager; ends the round.
\item[check] Declining to wager any chips when not facing a bet.
\item[chip] Marker representing value used for wagers; all wagers must be a whole numbers of chips.
\item[dealer] The player who puts the small blind into the pot.  Acts first on round 1, and second on the later rounds.  Traditionally, they would 
distribute public and private cards from the deck.  
\item[flop] The second round; can refer to either the 3 revealed public cards, or the betting round after these cards are revealed.
\item[fold] Give up on the current game, forfeiting all wagers placed in the pot.  Ends a player's participation in the game.
\item[hand] Many different meanings: the combination of the best 5 cards from the public cards and private cards, just the private cards themselves, or a single game of poker (for clarity, we avoid this final meaning). 
\item[milli-big-blinds per game (mbb/g)] Average winning rate over a number of games, measured in thousandths of big blinds.
\item[pot] The collected chips from all wagers.
\item[pre-flop] The first round; can refer to either the hole cards, or the betting round after these cards are distributed.
\item[private cards] Cards dealt face down, visible only to one player.  Used in combination with public cards to create a hand. Also called hole cards.
\item[public cards] Cards dealt face up, visible to all players.  Used in combination with private cards to create a hand. Also called community cards.
\item[raise] Increasing the size of a wager in a round, putting more chips into the pot than is required to call the current bet.
\item[river] The fourth and final round; can refer to either the 1 revealed public card, or the betting round after this card is revealed.
\item[showdown] After the river, players who have not folded show their private cards to determine the player with the best hand.  The player with the best hand takes all of the chips in the pot.
\item[small blind] Initial wager made by the dealer before any cards are dealt.  The small blind is half the size of the big blind.
\item[stack] The maximum amount of chips a player can wager or win in a single game.
\item[turn] The third round; can refer to either the 1 revealed public card, or the betting round after this card is revealed.
\end{description}

\section*{Performance Against Professional Players}

To assess DeepStack relative to expert humans, players were
recruited with assistance from the International Federation of Poker~\cite{IFP} to identify and recruit professional poker players through their 
member nation organizations.  We only selected participants from those who self-identified as a ``professional poker player'' during registration.
Players were given four weeks to complete a 3,000 game match.  To incentivize
players, monetary prizes of \$5,000, \$2,500, and \$1,250 (CAD) were awarded to
the top three players (measured by AIVAT) that completed their match.  
The participants were informed of all of these details when they registered to participate.
Matches
were played between November 7th and December 12th, 2016, and run using an
online user interface~\cite{DustinGUISource} where players had the option to play up to
four games simultaneously as is common in online poker sites.  A total of 33 players from 17 countries played
against DeepStack.  DeepStack's performance against each individual is
presented in Table~\ref{tab-human}, with complete game histories available as part of the supplementary online materials. 

\begin{table}[!hp]
\centering
\small
\caption{Results against professional poker players estimated with AIVAT (Luck Adjusted Win Rate) and chips won (Unadjusted Win Rate), both measured in mbb/g.  Recall 10mbb/g equals 1bb/100.  Each estimate
is followed by a 95\% confidence interval.  $\ddagger$ marks a participant who completed the 3000 games after their allotted four week period.}

\label{tab-human}

\fboxrule0.2pt
\fboxsep0pt

\begin{tabular}{m{0.21\textwidth}m{0.03\textwidth}|rrr@{~$\pm$~}R{0.0485\textwidth}r@{~$\pm$~}R{0.0605\textwidth}}
\toprule
Player & & Rank & Hands & \multicolumn{2}{r}{\begin{tabular}{c}Luck Adjusted\\Win Rate\end{tabular}} & \multicolumn{2}{r}{\begin{tabular}{c}Unadjusted\\Win Rate\end{tabular}} \\
\midrule
Martin Sturc & \fbox{\includegraphics[width=\linewidth]{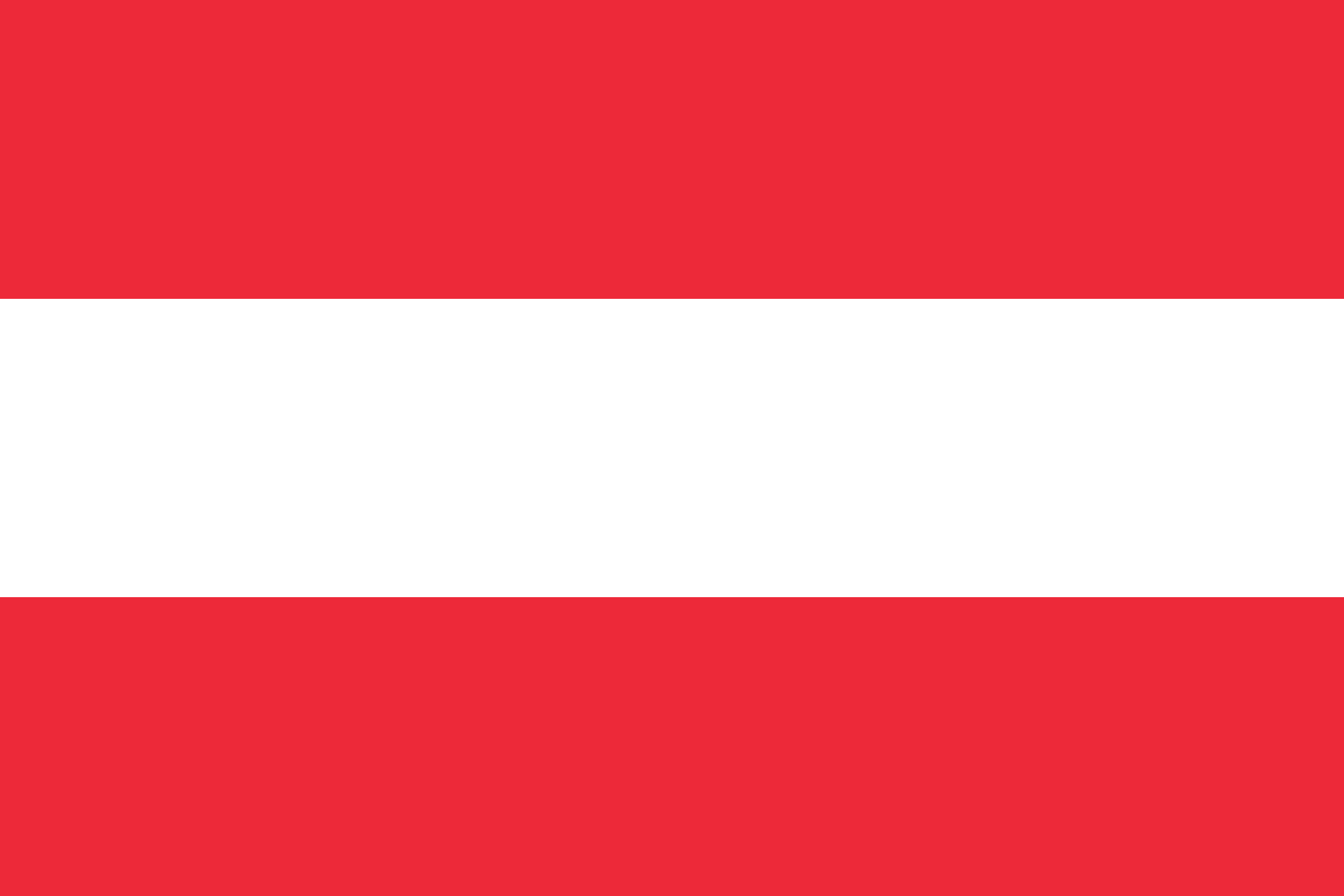}} & 1 & 3000  & $70$ & $119$ & $-515$ & $575$ \\
Stanislav Voloshin & \fbox{\includegraphics[width=\linewidth]{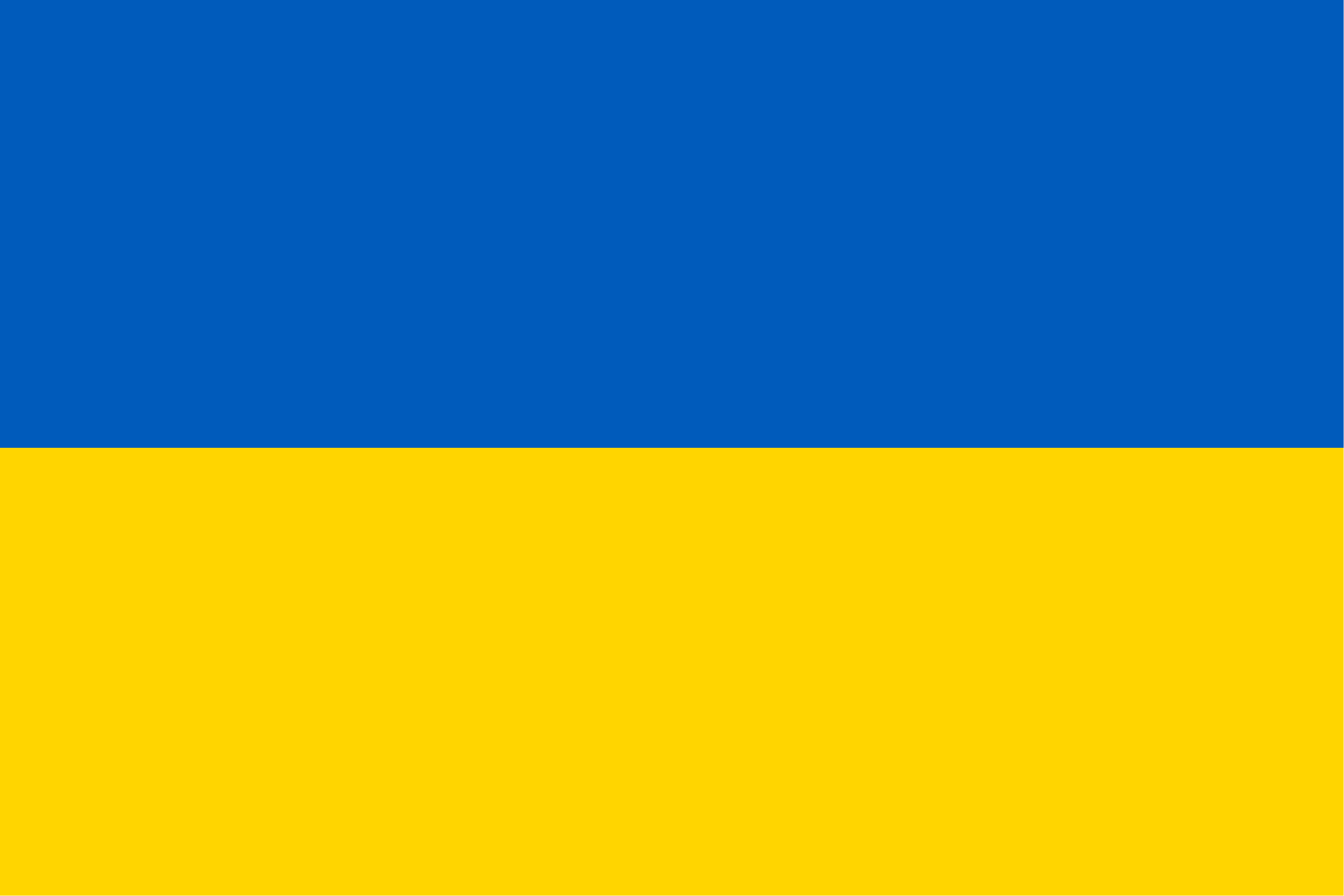}} & 2 & 3000  & $126$ & $103$ & $-65$ & $648$ \\
Prakshat Shrimankar & \fbox{\includegraphics[width=\linewidth]{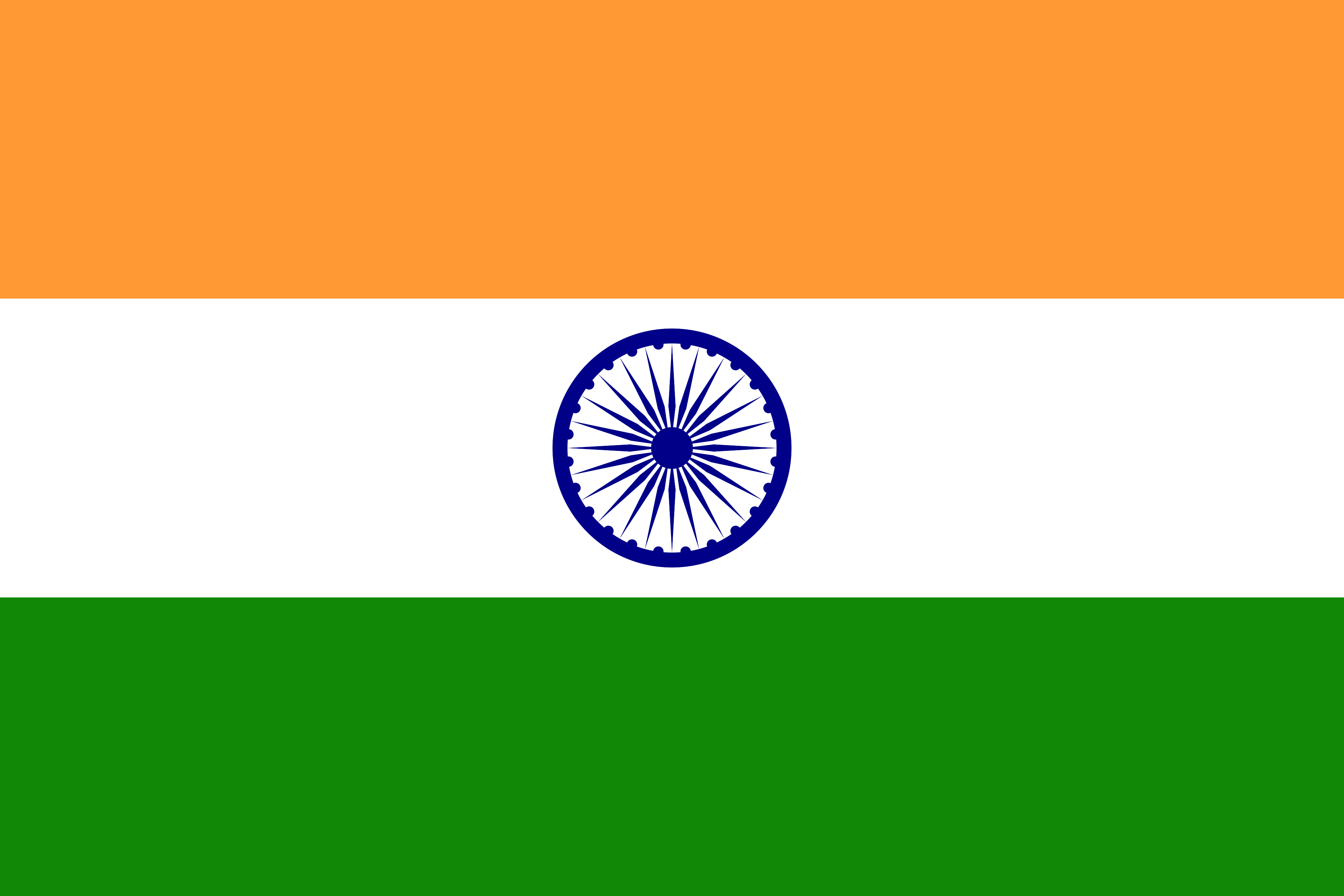}} & 3 & 3000  & $139$ & $97$ & $174$ & $667$ \\
Ivan Shabalin & \fbox{\includegraphics[width=\linewidth]{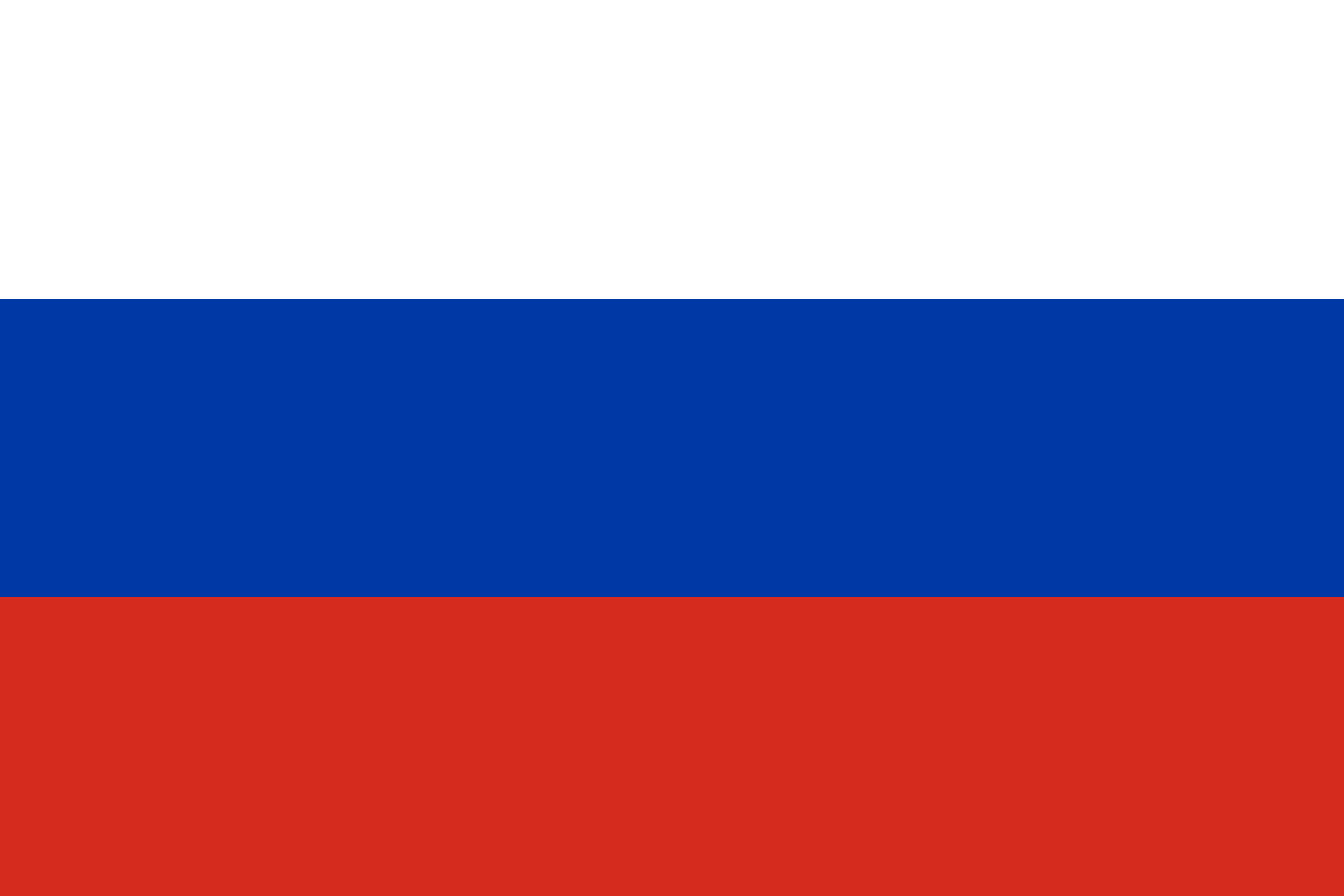}} & 4 & 3000  & $170$ & $99$ & $153$ & $633$ \\
Lucas Schaumann & \fbox{\includegraphics[width=\linewidth]{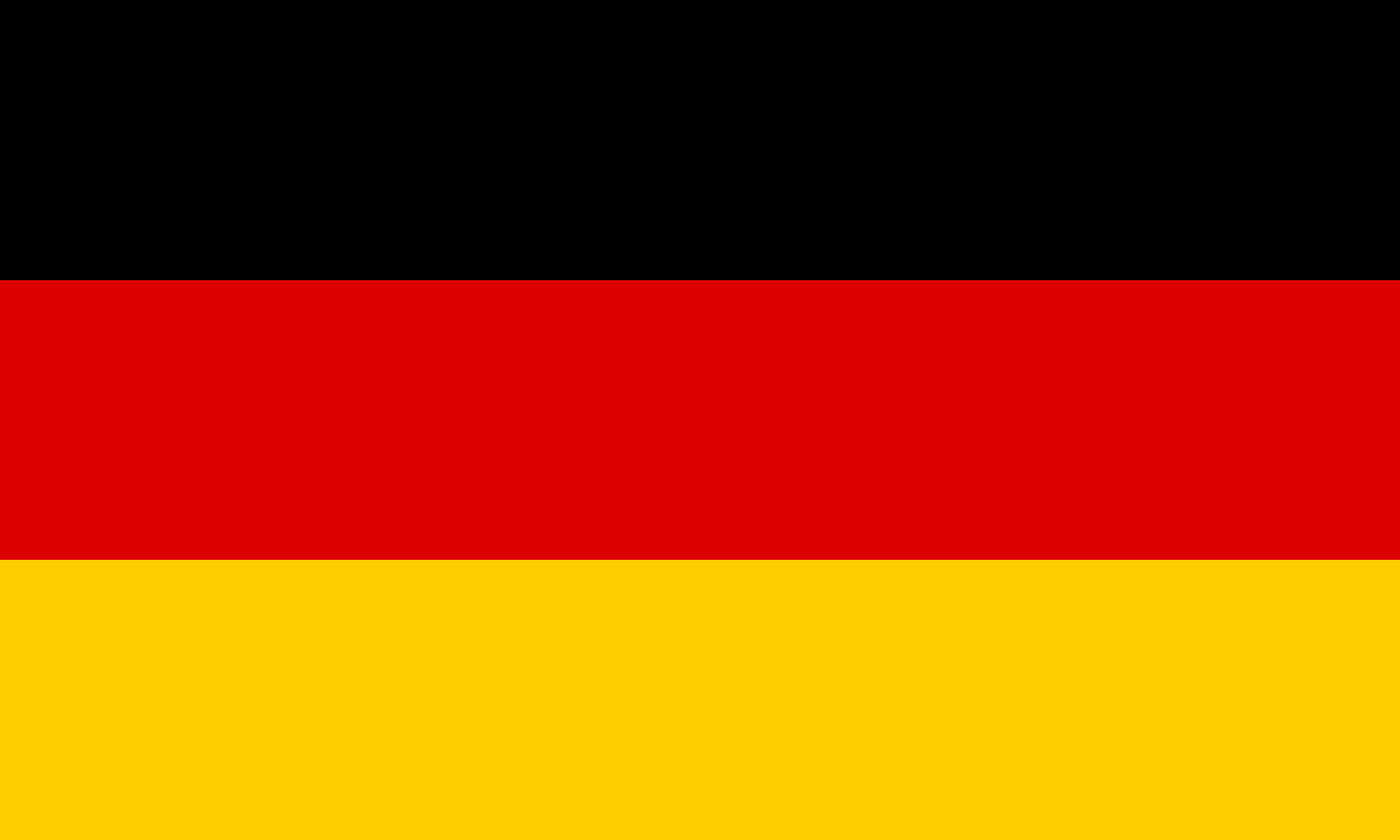}} & 5 & 3000  & $207$ & $87$ & $160$ & $576$ \\
Phil Laak & \fbox{\includegraphics[width=\linewidth]{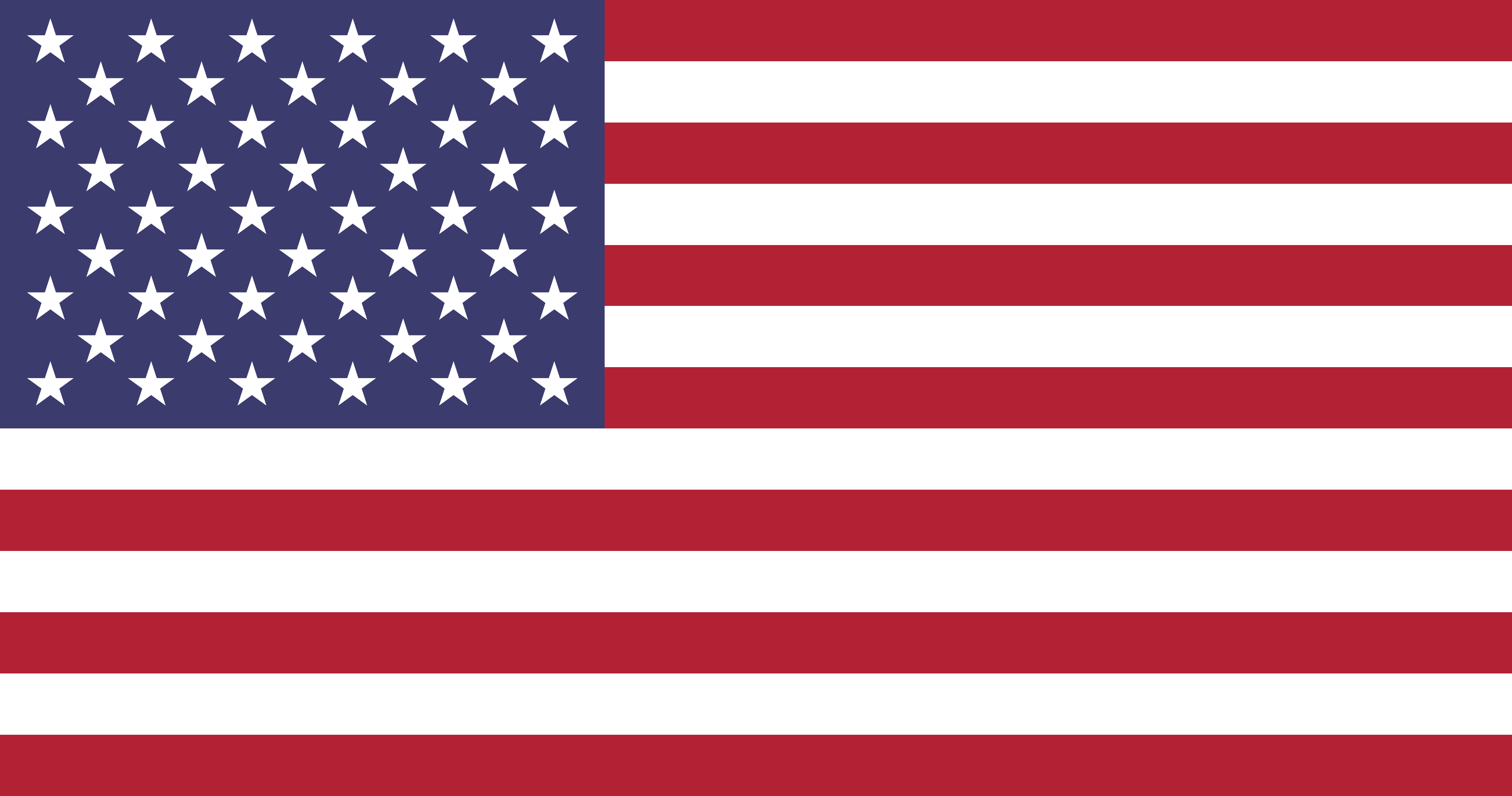}} & 6 & 3000  & $212$ & $143$ & $774$ & $677$ \\
Kaishi Sun & \fbox{\includegraphics[width=\linewidth]{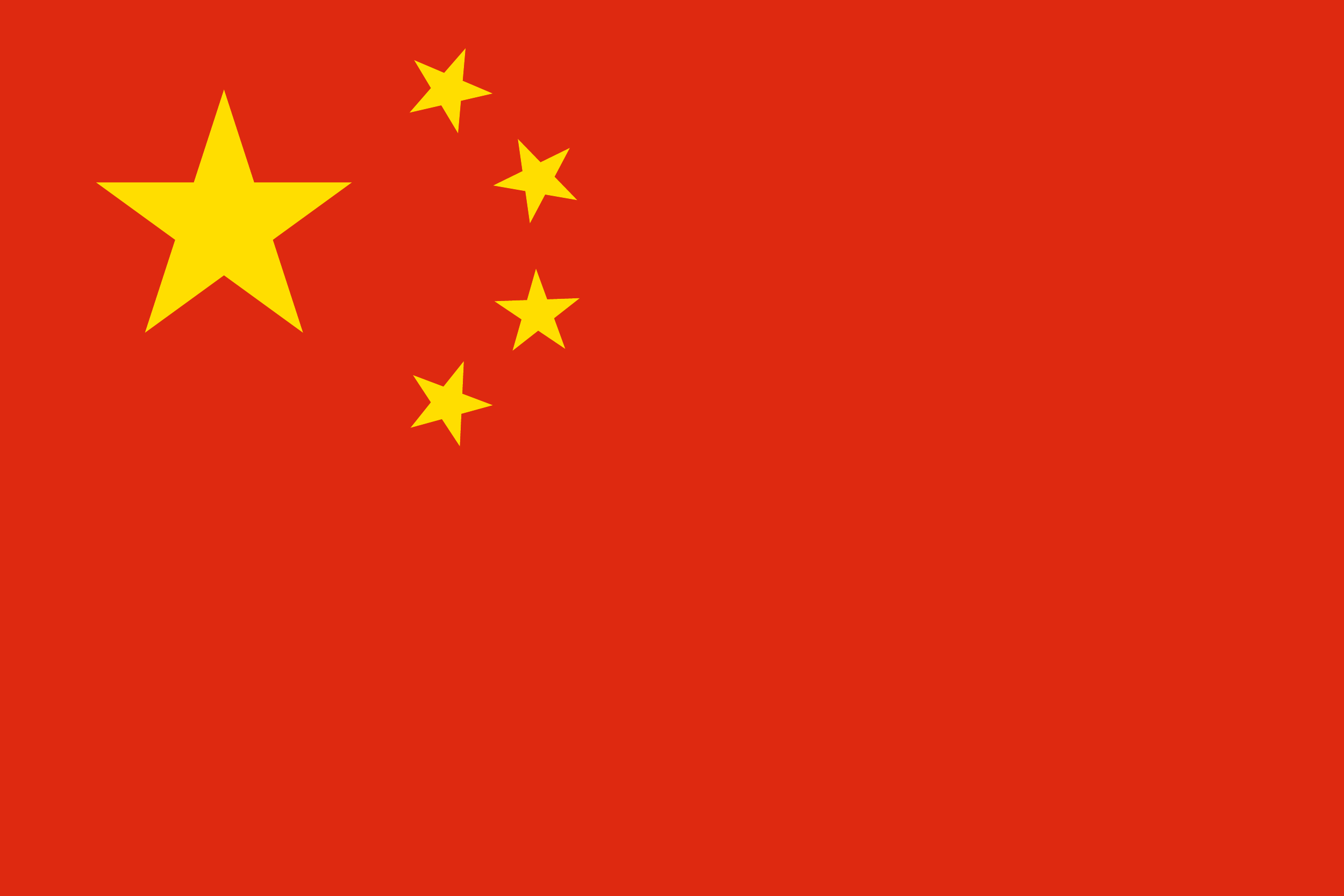}} & 7 & 3000  & $363$ & $116$ & $5$ & $729$ \\
Dmitry Lesnoy & \fbox{\includegraphics[width=\linewidth]{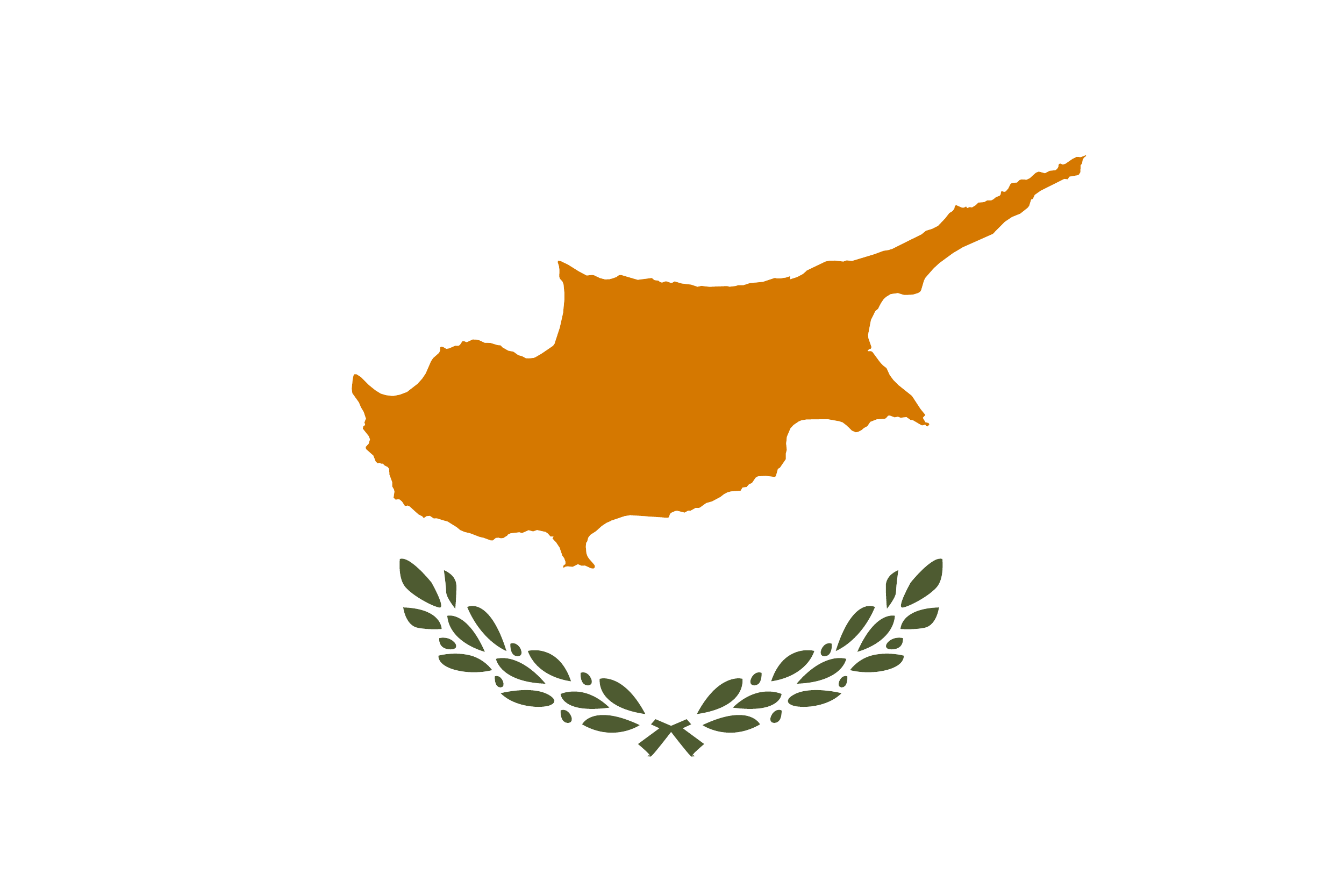}} & 8 & 3000  & $411$ & $138$ & $-87$ & $753$ \\
Antonio Parlavecchio & \fbox{\includegraphics[width=\linewidth]{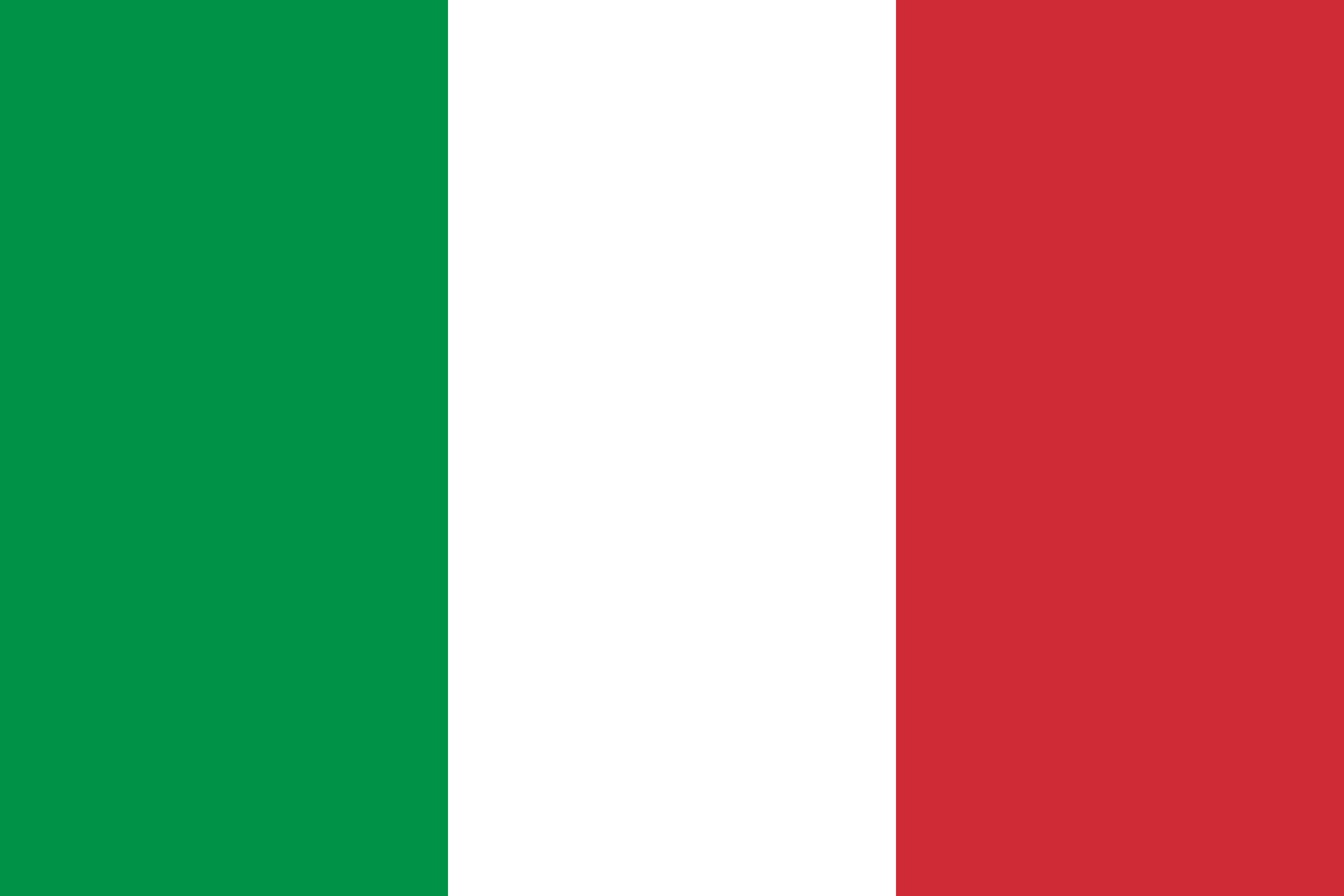}} & 9 & 3000  & $618$ & $212$ & $1096$ & $962$ \\
Muskan Sethi & \fbox{\includegraphics[width=\linewidth]{figs/flags/wiki/India}} & 10 & 3000  & $1009$ & $184$ & $2144$ & $1019$ \\
\midrule
\midrule
Pol Dmit\textsuperscript{$\ddagger$} & \fbox{\includegraphics[width=\linewidth]{figs/flags/wiki/Russia}} & -- & 3000  & $1008$ & $156$ & $883$ & $793$ \\
Tsuneaki Takeda & \fbox{\includegraphics[width=\linewidth]{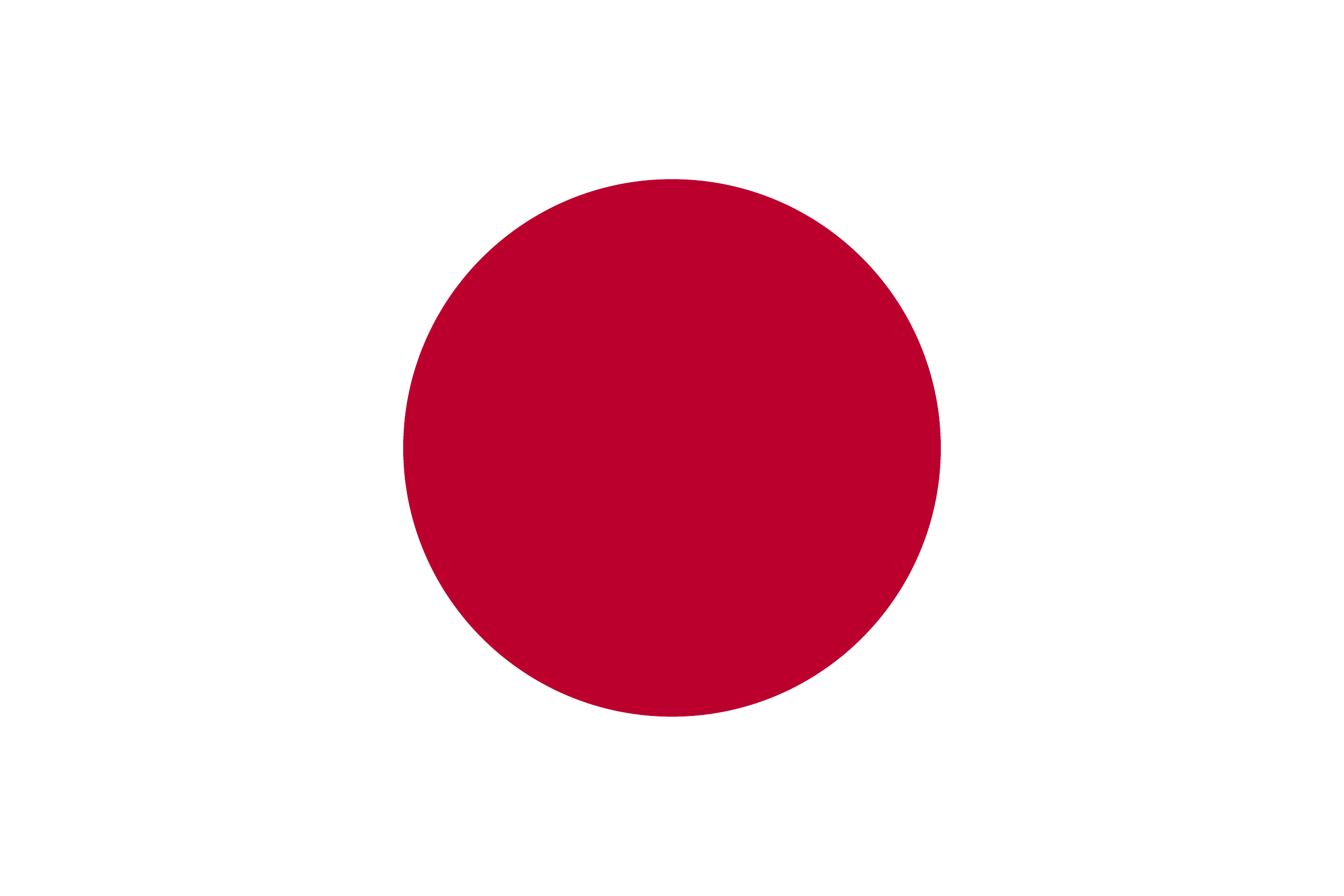}} & -- & 1901  & $628$ & $231$ & $-332$ & $1228$ \\
Youwei Qin & \fbox{\includegraphics[width=\linewidth]{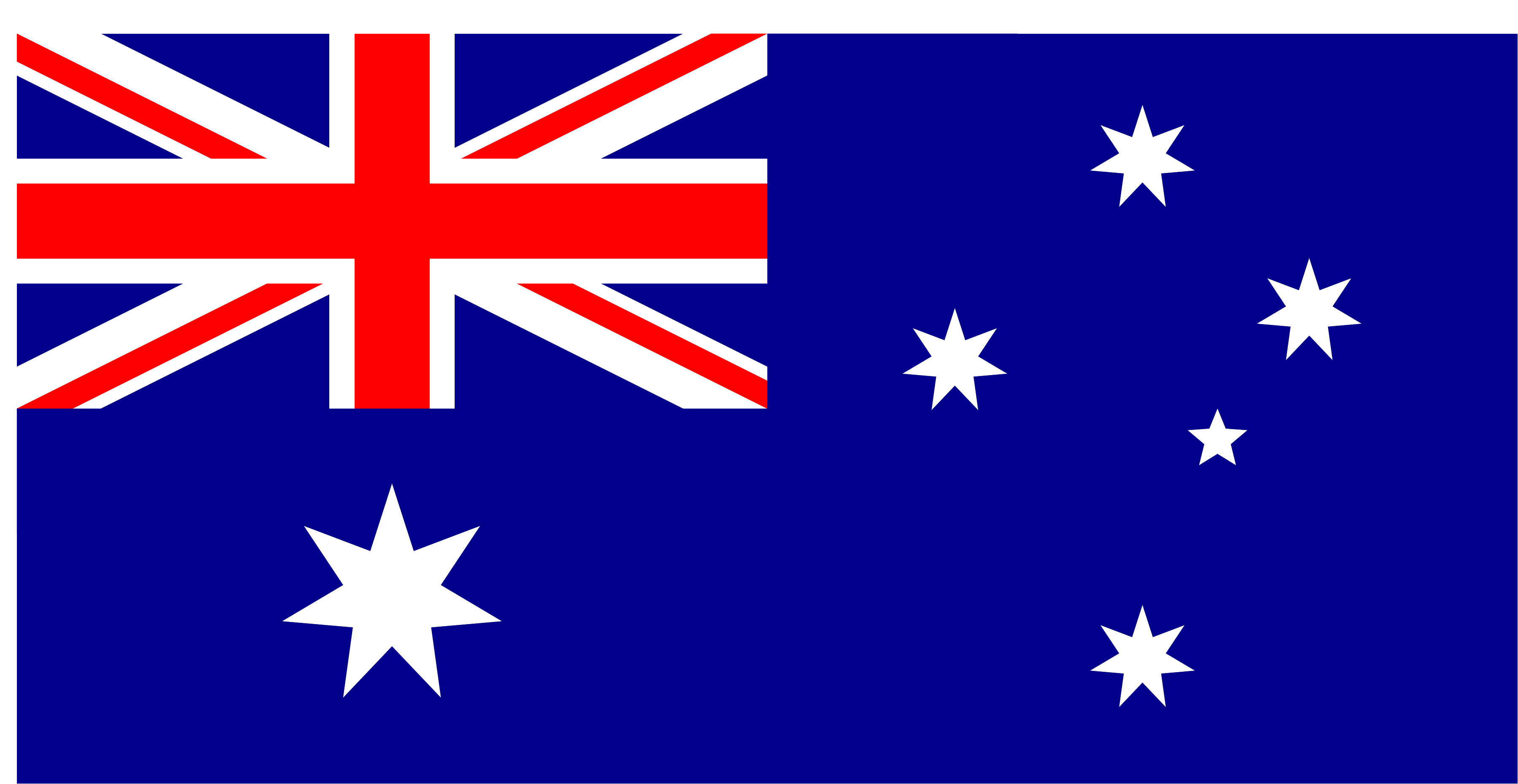}} & -- & 1759  & $1311$ & $331$ & $1958$ & $1799$ \\
Fintan Gavin & \fbox{\includegraphics[width=\linewidth]{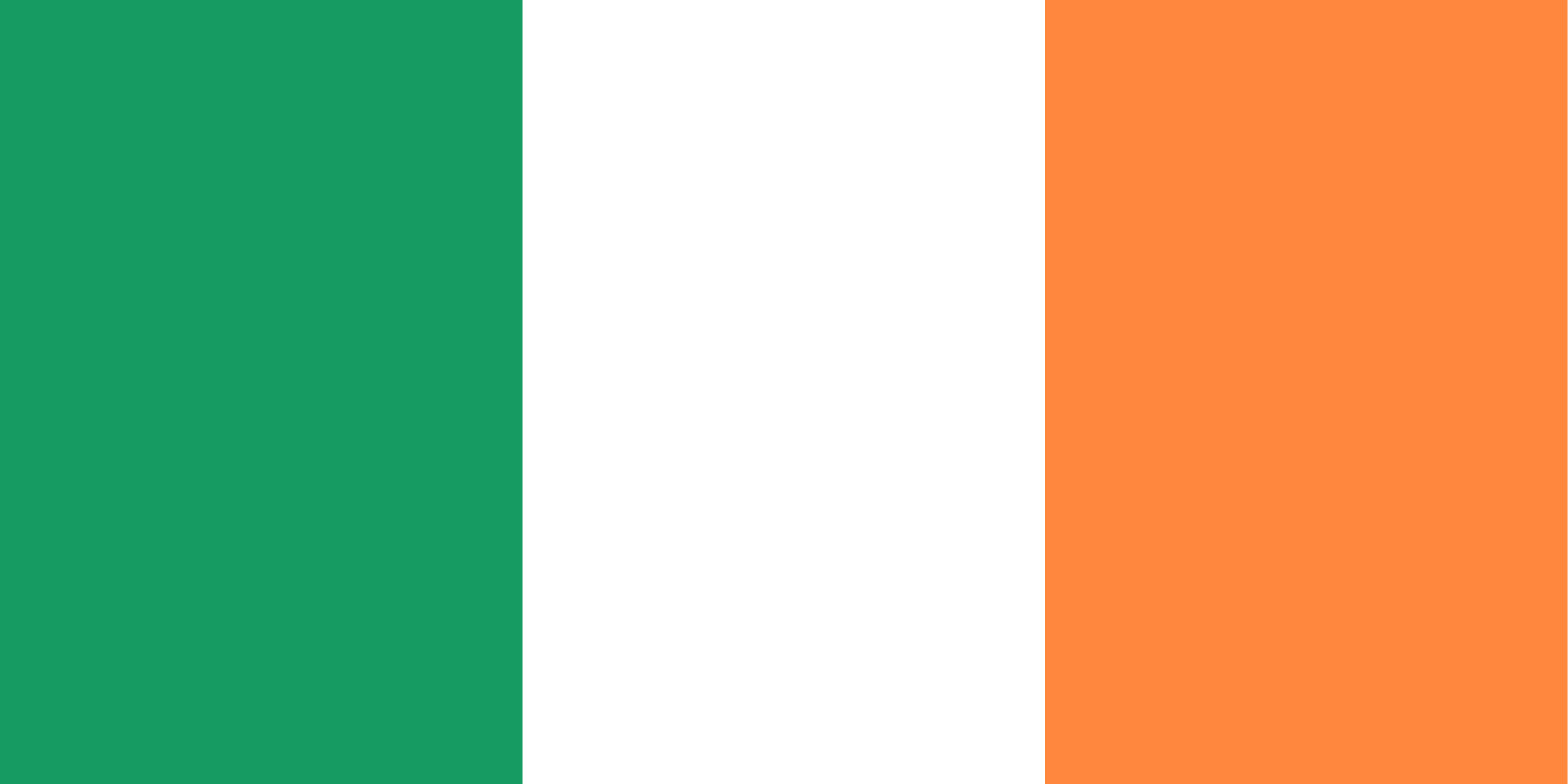}} & -- & 1555  & $635$ & $278$ & $-26$ & $1647$ \\
Giedrius Talacka & \fbox{\includegraphics[width=\linewidth]{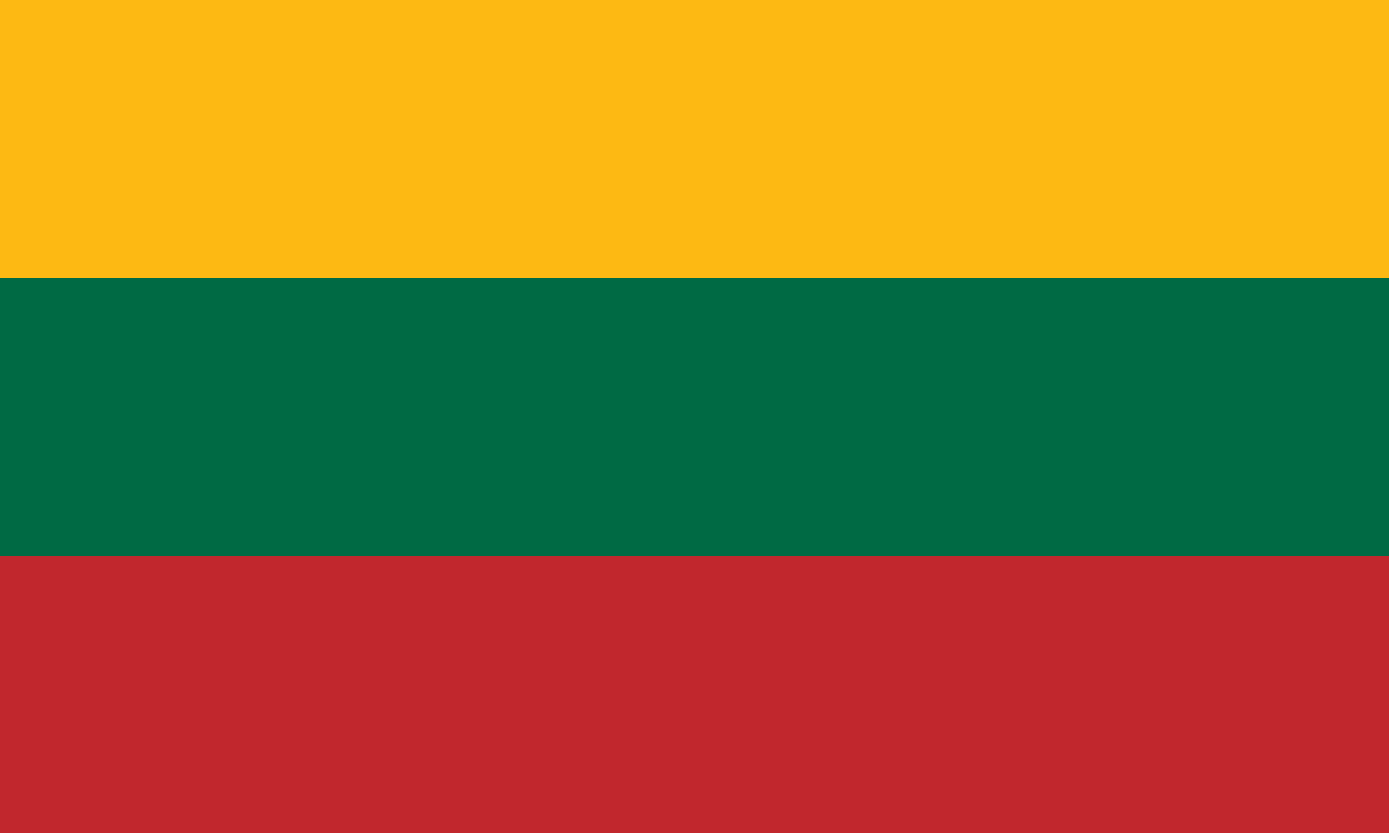}} & -- & 1514  & $1063$ & $338$ & $459$ & $1707$ \\
Juergen Bachmann & \fbox{\includegraphics[width=\linewidth]{figs/flags/wiki/Germany}} & -- & 1088  & $527$ & $198$ & $1769$ & $1662$ \\
Sergey Indenok & \fbox{\includegraphics[width=\linewidth]{figs/flags/wiki/Russia}} & -- & 852  & $881$ & $371$ & $253$ & $2507$ \\
Sebastian Schwab & \fbox{\includegraphics[width=\linewidth]{figs/flags/wiki/Germany}} & -- & 516  & $1086$ & $598$ & $1800$ & $2162$ \\
Dara O'Kearney & \fbox{\includegraphics[width=\linewidth]{figs/flags/wiki/Ireland}} & -- & 456  & $78$ & $250$ & $223$ & $1688$ \\
Roman Shaposhnikov & \fbox{\includegraphics[width=\linewidth]{figs/flags/wiki/Russia}} & -- & 330  & $131$ & $305$ & $-898$ & $2153$ \\
Shai Zurr & \fbox{\includegraphics[width=\linewidth]{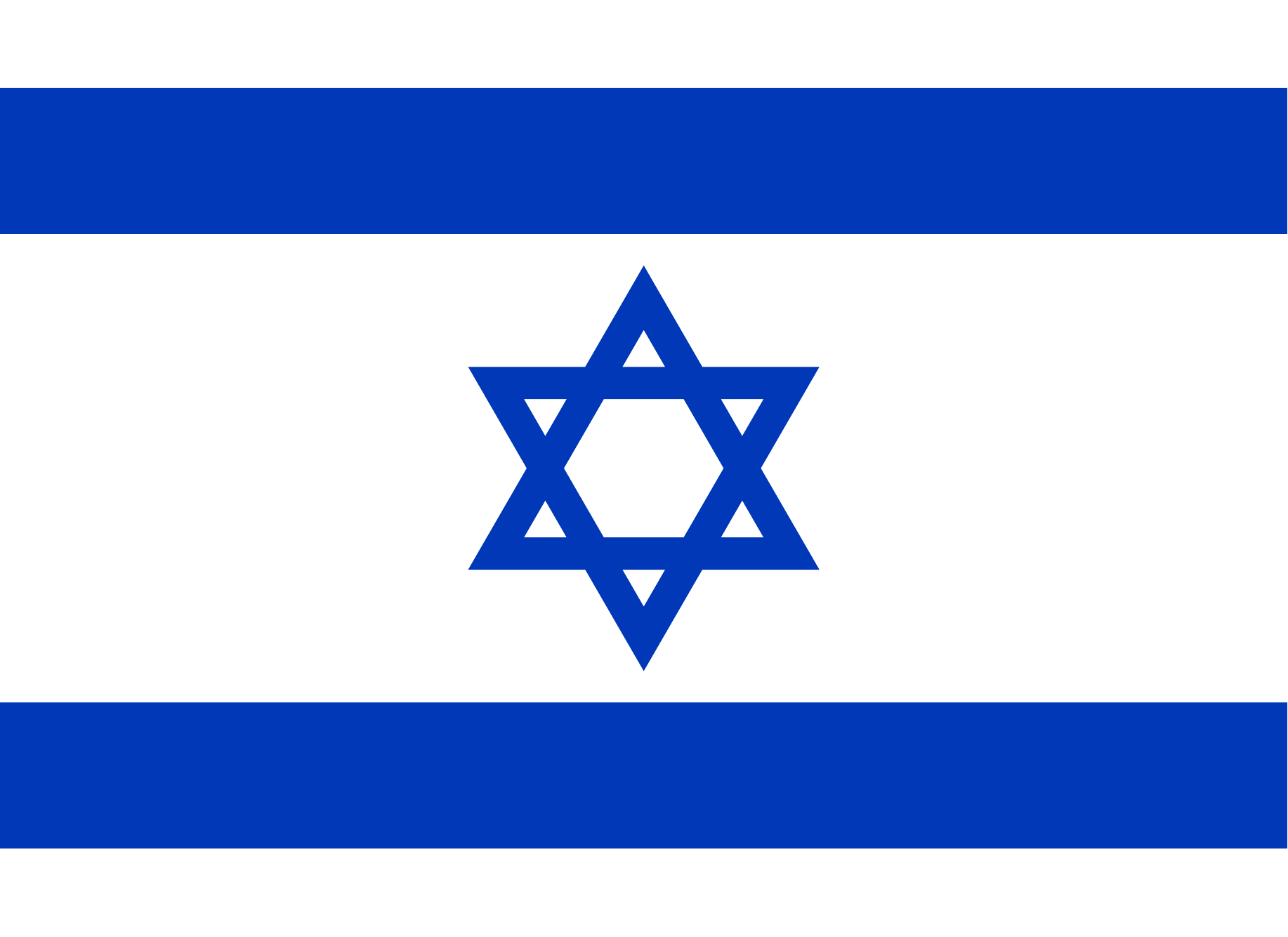}} & -- & 330  & $499$ & $360$ & $1154$ & $2206$ \\
Luca Moschitta & \fbox{\includegraphics[width=\linewidth]{figs/flags/wiki/Italy}} & -- & 328  & $444$ & $580$ & $1438$ & $2388$ \\
Stas Tishekvich & \fbox{\includegraphics[width=\linewidth]{figs/flags/wiki/Israel}} & -- & 295  & $-45$ & $433$ & $-346$ & $2264$ \\
Eyal Eshkar & \fbox{\includegraphics[width=\linewidth]{figs/flags/wiki/Israel}} & -- & 191  & $18$ & $608$ & $715$ & $4227$ \\
Jefri Islam & \fbox{\includegraphics[width=\linewidth]{figs/flags/wiki/Austria}} & -- & 176  & $997$ & $700$ & $3822$ & $4834$ \\
Fan Sun & \fbox{\includegraphics[width=\linewidth]{figs/flags/wiki/China}} & -- & 122  & $531$ & $774$ & $-1291$ & $5456$ \\
Igor Naumenko & \fbox{\includegraphics[width=\linewidth]{figs/flags/wiki/Ukraine}} & -- & 102  & $-137$ & $638$ & $851$ & $1536$ \\
Silvio Pizzarello & \fbox{\includegraphics[width=\linewidth]{figs/flags/wiki/Italy}} & -- & 90  & $1500$ & $2100$ & $5134$ & $6766$ \\
Gaia Freire & \fbox{\includegraphics[width=\linewidth]{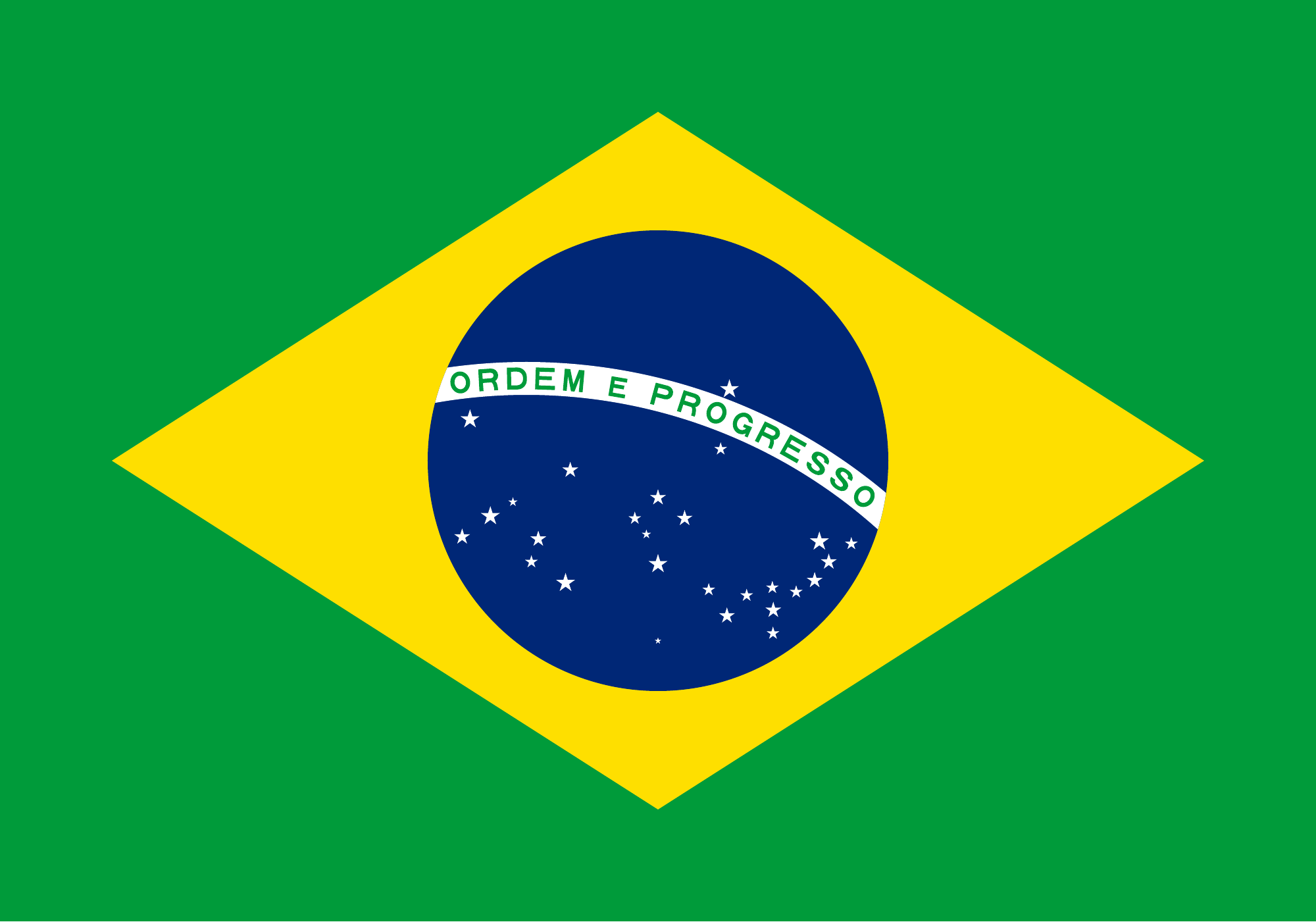}} & -- & 76  & $369$ & $136$ & $138$ & $694$ \\
Alexander B\"{o}s & \fbox{\includegraphics[width=\linewidth]{figs/flags/wiki/Austria}} & -- & 74  & $487$ & $756$ & $1$ & $2628$ \\
Victor Santos & \fbox{\includegraphics[width=\linewidth]{figs/flags/wiki/Brazil}} & -- & 58  & $475$ & $462$ & $-1759$ & $2571$ \\
Mike Phan & \fbox{\includegraphics[width=\linewidth]{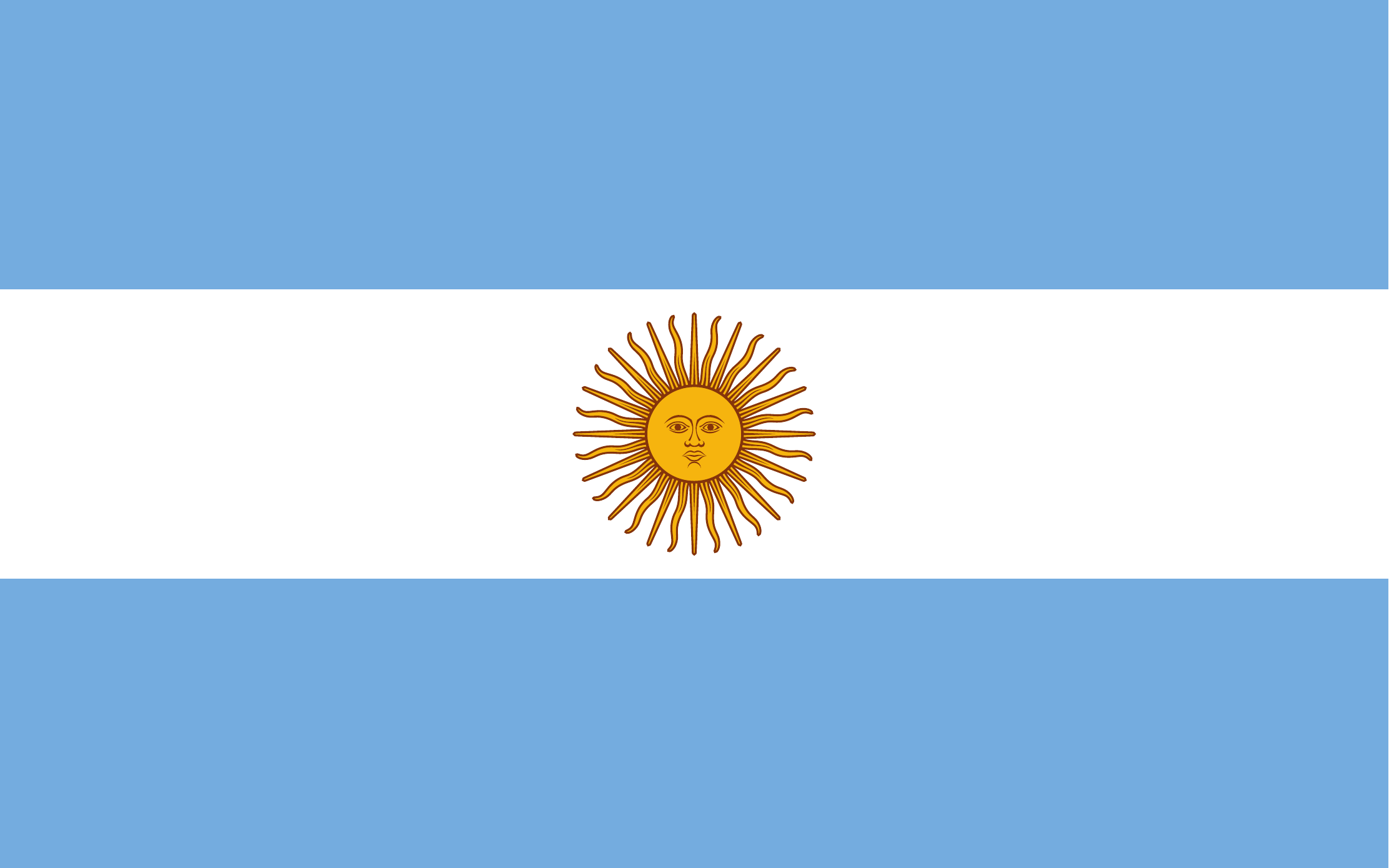}} & -- & 32  & $-1019$ & $2352$ & $-11223$ & $18235$ \\
Juan Manuel Pastor & \fbox{\includegraphics[width=\linewidth]{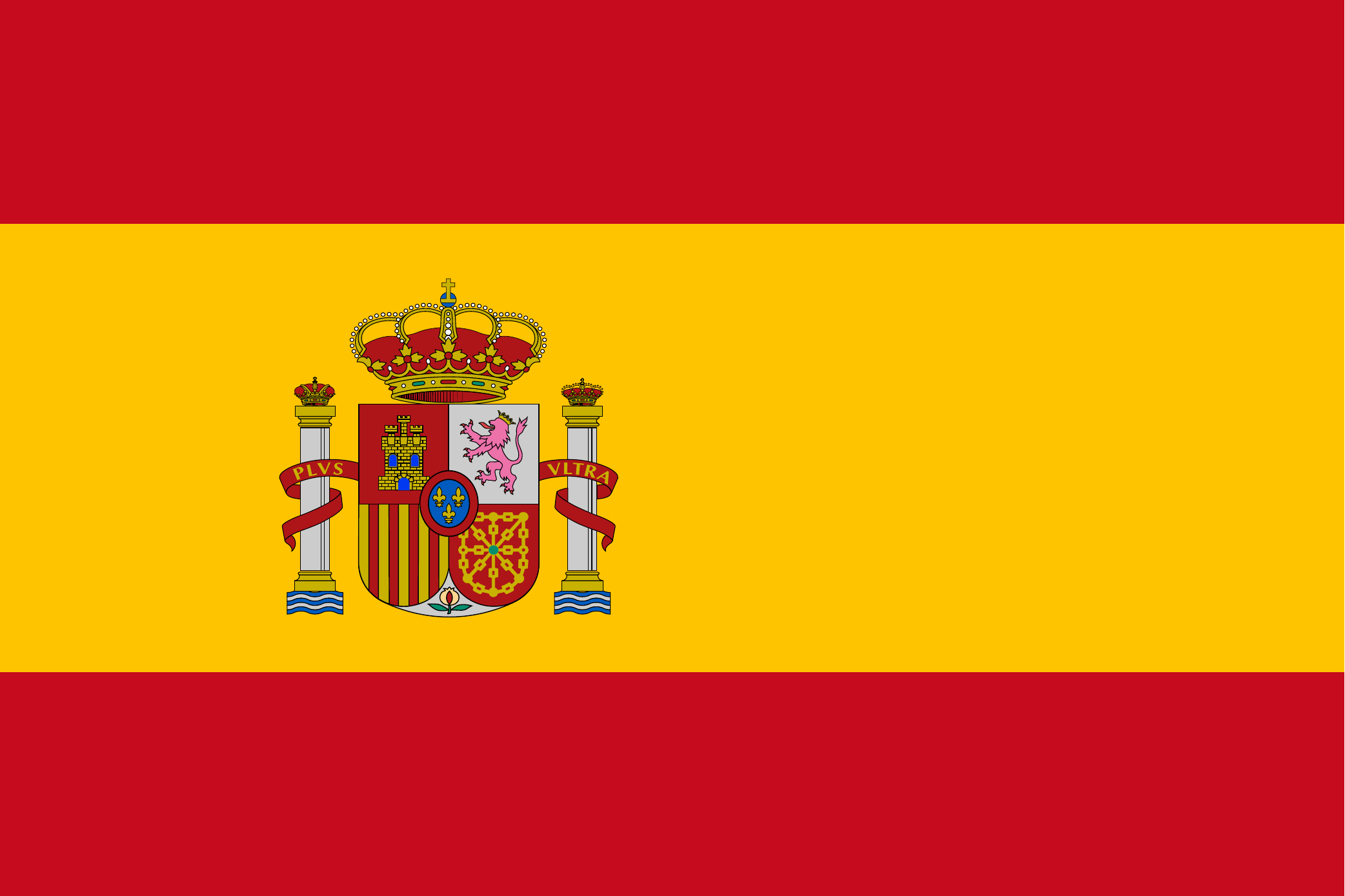}} & -- & 7  & $2744$ & $3521$ & $7286$ & $9856$ \\
\midrule
Human Professionals &  & & 44852 & $486$ & $40$ & $492$ & $220$ \\
\bottomrule
\end{tabular}
\end{table}

\section*{Local Best Response of DeepStack}

The goal of DeepStack, and much of the work on AI in poker, is to approximate a Nash equilibrium, i.e., produce a strategy with low exploitability.  The size of HUNL makes an explicit best-response computation intractable and so exact exploitability cannot be measured.  A common alternative is to play two strategies against each other.
However, head-to-head performance in imperfect information games has repeatedly been shown to be a poor estimation of equilibrium approximation quality.  For example, consider an exact Nash equilibrium strategy in the game of Rock-Paper-Scissors playing against a strategy that almost always plays ``rock''.  The results are a tie, but their playing strengths in terms of exploitability are vastly different.  This same issue has been seen in heads-up limit Texas hold'em as well (Johanson, IJCAI 2011), where the relationship between head-to-head play and exploitability, which is tractable in that game, is indiscernible.  The introduction of local best response (LBR) as a technique for finding a lower-bound on a strategy's exploitability gives evidence of the same issue existing in HUNL.  Act1 and Slumbot (second and third place in the previous ACPC) were statistically indistinguishable in head-to-head play (within 20 mbb/g), but Act1 is 1300mbb/g less exploitable as measured by LBR.  This is why we use LBR to evaluate DeepStack.

LBR is a simple, yet powerful, technique to produce a lower bound on a strategy's exploitability in HUNL~\cite{Lisy17:LocalBR} .  It explores a fixed set of options to find a ``locally'' good action against the strategy.  While it seems natural that more options would be better, this is not always true.  More options may cause it to find a locally good action that misses out on a future opportunity to exploit an even larger flaw in the opponent.  In fact, LBR sometimes results in larger lower bounds when not considering any bets in the early rounds, so as to increase the size of the pot and thus the magnitude of a strategy's future mistakes.  LBR was recently used to show that abstraction-based agents are significantly exploitable (see Table~\ref{tab-localbr}). The first three strategies are submissions from recent Annual Computer Poker Competitions. They all use both card and action abstraction and were found to be even more exploitable than simply folding every game in all tested cases.  The strategy ``Full Cards'' does not use any card abstraction, but uses only the sparse fold, call, pot-sized bet, all-in betting abstraction using hard translation~\cite{Schnizlein09:Translation}. Due to computation and memory requirements, we computed this strategy only for a smaller stack of 100 big blinds. Still, this strategy takes almost 2TB of memory and required approximately 14 CPU years to solve. Naturally, it cannot be exploited by LBR within the betting abstraction, but it is heavily exploitable in settings using other betting actions that require it to translate its opponent's actions, again losing more than if it folded every game.

As for DeepStack, under all tested settings of LBR's available actions, it fails to find any exploitable flaw.  In fact, it is losing 350 mbb/g  or more to DeepStack.  Of particular interest is the final column aimed to exploit DeepStack's flop strategy.  The flop is where DeepStack is most dependent on its counterfactual value networks to provide it estimates through the end of the game.
While these experiments do not prove that DeepStack is flawless, it does suggest its flaws require a more sophisticated search procedure than what is needed to exploit abstraction-based programs.

\begin{table}
\centering
\caption{Exploitability lower bound of different programs using local best response (LBR). LBR evaluates only the listed actions in each round as shown in each row.  F, C, P, A, refer to fold, call, a pot-sized bet, and all-in, respectively. 56bets includes the actions fold, call and 56 equidistant pot fractions as defined in the original LBR paper \cite{Lisy17:LocalBR}.  $\ddagger$: Always Fold checks when not facing a bet, and so it cannot be maximally exploited without a betting action.}
\label{tab-localbr}

\begin{tabular}{cc|rrrr}
\toprule
\multicolumn{2}{c}{~} & \multicolumn{4}{c}{Local best response performance (mbb/g)}\\
\midrule
\multirow{4}{*}{LBR Settings} & Pre-flop & \multicolumn{1}{c}{F, C} & \multicolumn{1}{c}{C} & \multicolumn{1}{c}{C} & \multicolumn{1}{c}{C}\\
 & Flop & \multicolumn{1}{c}{F, C} & \multicolumn{1}{c}{C} & \multicolumn{1}{c}{C} & \multicolumn{1}{c}{56bets}\\ 
 & Turn & \multicolumn{1}{c}{F, C} & \multicolumn{1}{c}{F, C, P, A} & \multicolumn{1}{c}{56bets} & \multicolumn{1}{c}{F, C}\\
 & River& \multicolumn{1}{c}{F, C} & \multicolumn{1}{c}{F, C, P, A} & \multicolumn{1}{c}{56bets} & \multicolumn{1}{c}{F, C}\\
\midrule
\multicolumn{2}{c|}{Hyperborean (2014)} & 721 $\pm$ 56 & 3852 $\pm$ 141 & 4675 $\pm$ 152 & 983 $\pm$ ~~95 \\
\multicolumn{2}{c|}{Slumbot  (2016)} & 522 $\pm$ 50 & 4020 $\pm$ 115 & 3763 $\pm$ 104 & 1227 $\pm$ ~~79  \\
\multicolumn{2}{c|}{Act1 (2016)} & 407 $\pm$ 47 & 2597 $\pm$ 140 & 3302 $\pm$ 122 &  847 $\pm$ ~~78\\ 
\multicolumn{2}{c|}{Always Fold} & $\ddagger$250 $\pm$ ~~0 & 750 $\pm$ ~~~~0 & 750 $\pm$ ~~~~0 & 750 $\pm$ ~~~~0 \\
\multicolumn{2}{c|}{Full Cards [100 BB]} &  -424 $\pm$ 37 & -536 $\pm$ ~~87 & 2403 $\pm$ ~~87  & 1008 $\pm$ ~~68 \\
\multicolumn{2}{c|}{\cellcolor[HTML]{C0C0C0} DeepStack} & \cellcolor[HTML]{C0C0C0}-428 $\pm$ 87 & \cellcolor[HTML]{C0C0C0}-383 $\pm$ 219 &  \cellcolor[HTML]{C0C0C0}-775 $\pm$ 255  &  \cellcolor[HTML]{C0C0C0}-602 $\pm$ 214\\
\bottomrule
\end{tabular}
\end{table}

\section*{DeepStack Implementation Details}

Here we describe the specifics for how DeepStack employs continual re-solving and how its deep counterfactual value networks were trained.

\subsection*{Continual Re-Solving}

As with traditional re-solving, the re-solving step of the DeepStack algorithm solves an augmented game.  The augmented game is designed to produce a strategy for the player such that the bounds for the opponent's counterfactual values are satisfied.  DeepStack uses a modification of the original CFR-D gadget~\cite{cprg:cfrd} for its augmented game, as discussed below.  While the max-margin gadget~\cite{Moravcik16:Subgames} is designed to improve the performance of poor strategies for abstracted agents near the end of the game, the CFR-D gadget performed better in early testing.

The algorithm DeepStack uses to solve the augmented game is a hybrid of vanilla CFR  \cite{ZinkevichEtAl07} and CFR$^+$ \cite{Tammelin15:CFR+}, which uses regret matching$^+$ like CFR$^+$, but does uniform weighting and simultaneous updates like vanilla CFR.  When computing the final average strategy and average counterfactual values, we omit the early iterations of CFR in the averages.

A major design goal for DeepStack's implementation was to typically play at least as fast as a human would using commodity hardware and a single GPU.
The degree of lookahead tree sparsity and the number of re-solving iterations are the principle decisions that we tuned to achieve this goal.
These properties were chosen separately for each round to achieve a consistent speed on each round. Note that DeepStack has no fixed requirement on the density of its lookahead tree besides those imposed by hardware limitations and speed constraints.

The lookahead trees vary in the actions available to the player acting, the actions available for the opponent's response, and the actions available to either player for the remainder of the round. 
We use the end of the round as our depth limit, except on the turn when the remainder of the game is solved.  
On the pre-flop and flop, we use trained counterfactual value networks to return values after the flop or turn card(s) are revealed.
Only applying our value function to public states at the start of a round is particularly convenient in that that we don't need to include the bet faced as an input to the function.  
Table~\ref{tab-lookahead} specifies lookahead tree properties for each round. 

\def\halfP{\textonehalf{}P}
\def\twoP{2P}

\begin{table}[tb]
\centering
\caption{Lookahead re-solving specifics by round. The abbreviations of F, C, \halfP, P, \twoP, and A refer to fold, call, half of a pot-sized bet, a pot-sized bet, twice a pot-sized bet, and all in, respectively.  The final column specifies which neural network was used when the depth limit was exceeded: the flop, turn, or the auxiliary network. \label{tab-lookahead}}
{\small\begin{tabular}{l|llllll}
\toprule
           & CFR          & Omitted    & First    & Second & Remaining & NN \\
Round & Iterations & Iterations & Action & Action & Actions       & Eval \\
\midrule
Pre-flop & 1000  & 980 & F, C, \halfP, P, A & F, C, \halfP, P, \twoP, A & F, C, P, A & Aux/Flop \\
Flop &  1000 & 500 & F, C, \halfP, P, A & F, C, P, A & F, C, P, A & Turn\\
Turn & 1000 & 500 & F, C, \halfP, P, A & F, C, P, A & F, C, P, A  & --- \\
River & 2000  & 1000 & F, C, \halfP, P, \twoP, A & F, C, \halfP, P, \twoP, A & F, C, P, A & --- \\
 \bottomrule
\end{tabular}}
\end{table}

The pre-flop round is particularly expensive as it requires enumerating all 22,100 possible public cards on the flop and evaluating each with the flop network.  To speed up pre-flop play, we trained an additional auxiliary neural network to estimate the expected value of the flop network over all possible flops.  However, we only used this network during the initial omitted iterations of CFR.  During the final iterations used to compute the average strategy and counterfactual values, we did the expensive enumeration and flop network evaluations.  Additionally, we cache the re-solving result for every observed pre-flop situation.  When the same betting sequence occurs again, we simply reuse the cached results rather than recomputing.  For the turn round, we did not use a neural network after the final river card, but instead solved to the end of the game.  However, we used a bucketed abstraction for all actions on the river.  For acting on the river, the re-solving includes the remainder of the game and so no counterfactual value network was used.

\paragraph*{Actions in Sparse Lookahead Trees.} 
DeepStack's sparse lookahead trees use only a small subset of the game's possible actions. The first layer of actions immediately after the current public state defines the options considered for DeepStack's next action. The only purpose of the remainder of the tree is to estimate counterfactual values for the first layer during the CFR algorithm. Table~\ref{tab:cfvs} presents how well counterfactual values can be estimated using sparse lookahead trees with various action subsets. 

\begin{table}[tb]
\centering
\caption{Absolute ($L_1$), Euclidean ($L_2$), and maximum absolute ($L_{\infty}$) errors, in mbb/g, of counterfactual values computed with 1,000 iterations of CFR on sparse trees, averaged over 100 random river situations.  
The ground truth values were estimated by solving the game with 9 betting options and 4,000 iterations (first row).
}
\label{tab:cfvs}
\begin{tabular}{lr|rrr }
\toprule
 Betting & Size  & $L_1$  & $L_2$ & $L_\infty$ \\
  \midrule
 F, C, Min, \textonequarter{}P, \textonehalf{}P, \textthreequarters{}P, P, 2P, 3P, 10P, A [4,000 iterations]  & 555k & 0.0 & 0.0 & 0.0  \\
 F, C, Min, \textonequarter{}P, \textonehalf{}P, \textthreequarters{}P, P, 2P, 3P, 10P, A  & 555k & 18.06 & 0.891 & 0.2724 \\
  F, C, 2P, A & 48k & 64.79 & 2.672 & 0.3445\\
  F, C, \halfP, A & 100k & 58.24 & 3.426 & 0.7376 \\
  F, C, P, A & 61k & 25.51 & 1.272 & 0.3372 \\
  F, C, \halfP, P, A & 126k & 41.42 & 1.541 & 0.2955 \\
  F, C, P, 2P, A & 204k & 27.69 & 1.390 & 0.2543 \\
  F, C, \halfP, P, 2P, A & 360k & 20.96 & 1.059 & 0.2653 \\
 \bottomrule
\end{tabular}
\end{table}

The results show that the F, C, P, A, actions provide an excellent tradeoff between computational requirements via the size of the solved lookahead tree and approximation quality. Using more actions quickly increases the size of the lookahead tree, but does not substantially improve errors. Alternatively, using a single betting action that is not one pot has a small effect on the size of the tree, but causes a substantial error increase.

To further investigate the effect of different betting options, Table~\ref{tab-lbr_various_actions} presents the results of evaluating DeepStack with different action sets using LBR.
We used setting of LBR that proved most effective against the default set of DeepStack actions (see Table~\ref{tab-lookahead}).
While the extent of the variance in the 10,000 hand evaluation shown in Table~\ref{tab-lbr_various_actions} prevents us from declaring a best set of actions with certainty, the crucial point is that LBR is significantly losing to each of them, and that we can produce play that is difficult to exploit even choosing from a small number of actions.  Furthermore, the improvement of a small number of additional actions is not dramatic.

\begin{table}[tb]
\centering
\caption{Performance of LBR exploitation of DeepStack with different actions allowed on the first level of its lookahead tree using the best LBR configuration against the default version of DeepStack. LBR cannot exploit DeepStack regardless of its available actions.}
\label{tab-lbr_various_actions}
\begin{tabular}{lr}
\toprule
First level actions & LBR performance\\
\midrule
F, C, P, A & -479 $\pm$ 216\\
Default   & -383 $\pm$ 219\\
F, C, \halfP, P, 1\halfP, 2P, A & -406 $\pm$ 218\\
\bottomrule
\end{tabular}
\end{table}

\paragraph*{Opponent Ranges in Re-Solving.}
Continual re-solving does not require keeping track of the opponent's range.  The re-solving step essentially reconstructs a suitable range using the bounded counterfactual values.  In particular, the CFR-D gadget does this by giving the opponent the option, after being dealt a uniform random hand, of terminating the game (T) instead of following through with the game (F), allowing them to simply earn that hand's bound on its counterfactual value.  Only hands which are valuable to bring into the subgame will then be observed by the re-solving player.  However, this process of the opponent learning which hands to follow through with can make re-solving require many iterations.  An estimate of the opponent's range can be used to effectively warm-start the choice of opponent ranges, and help speed up the re-solving.

One conservative option is to replace the uniform random deal of opponent hands with any distribution over hands as long as it assigns non-zero probability to every hand.  For example, we could linearly combine an estimated range of the opponent from the previous re-solve (with weight $b$) and a uniform range (with weight $1-b$).  This approach still has the same theoretical guarantees as re-solving, but can reach better approximate solutions in fewer iterations.  Another option is more aggressive and sacrifices the re-solving guarantees when the opponent's range estimate is wrong.  It forces the opponent with probability $b$ to follow through into the game with a hand sampled from the estimated opponent range.  With probability $1-b$ they are given a uniform random hand and can choose to terminate or follow through.  This could prevent the opponent's strategy from reconstructing a correct range, but can speed up re-solving further when we have a good opponent range estimate.

DeepStack uses an estimated opponent range during re-solving only for the first action of a round, as this is the largest lookahead tree to re-solve.  The range estimate comes from the last re-solve in the previous round.  When DeepStack is second to act in the round, the opponent has already acted, biasing their range, so we use the conservative approach.  When DeepStack is first to act, though, the opponent could only have checked or called since our last re-solve.  Thus, the lookahead has an estimated range following their action.  So in this case, we use the aggressive approach.  In both cases, we set $b=0.9$.
 
\paragraph{Speed of Play.} 
The re-solving computation and neural network evaluations are both implemented in Torch7~\cite{collobert2011torch7} and run on a single NVIDIA GeForce GTX 1080 graphics card. 
This makes it possible to do fast batched calls to the counterfactual value networks for multiple public subtrees at once, which is key to making DeepStack fast.

Table \ref{tab-times} reports the average times between the end of the previous (opponent or chance) action and submitting the next action by both humans and DeepStack in our study.  DeepStack, on average, acted considerably faster than our human players.  It should be noted that some human players were playing up to four games simultaneously (although few players did more than two), and so the human players may have been focused on another game when it became their turn to act.

\begin{table}[tb]
\centering
\caption{Thinking times for both humans and DeepStack.
DeepStack's extremely fast pre-flop speed shows that pre-flop situations often resulted in cache hits.}
\label{tab-times}
\begin{tabular}{l|ll|ll}
\toprule
           & \multicolumn{4}{c}{Thinking Time (s)} \\
& \multicolumn{2}{c}{Humans} & \multicolumn{2}{c}{DeepStack} \\
Round & Median & Mean & Median & Mean \\
\midrule
Pre-flop & 10.3 & 16.2 & 0.04 & 0.2 \\
Flop & 9.1 & 14.6 & 5.9 & 5.9 \\
Turn & 8.0 & 14.0 & 5.4 & 5.5 \\
River & 9.5 & 16.2 & 2.2 & 2.1 \\
\midrule
Per Action & 9.6 & 15.4 & 2.3 & 3.0 \\
Per Hand & 22.0 & 37.4 & 5.7 & 7.2 \\
 \bottomrule
\end{tabular}
\end{table}

\subsection*{Deep Counterfactual Value Networks}

DeepStack uses two counterfactual value networks, one for the flop and one for the turn, as well as an auxiliary network that gives counterfactual values at the end of the pre-flop.  In order to train the networks, we generated random poker situations at the start of the flop and turn.  Each poker situation is defined by the pot size, ranges for both players, and dealt public cards.  The complete betting history is not necessary as the pot and ranges are a sufficient representation.  The output of the network are vectors of counterfactual values, one for each player.  The output values are interpreted as fractions of the pot size to improve generalization across poker situations. 

The training situations were generated by first sampling a pot size from a fixed distribution which was designed to approximate observed pot sizes from older HUNL programs.\footnote{The fixed distribution selects an interval from the set of intervals $\{[100, 100),$ $[200,400),$ $[400, 2000),$ $[2000, 6000),$ $[6000, 19950]\}$ with uniform probability, followed by uniformly selecting an integer from within the chosen interval.}
The player ranges for the training situations need to cover the space of possible ranges that CFR might encounter during re-solving, not just ranges that are likely part of a solution.  So we generated pseudo-random ranges that attempt to cover the space of possible ranges.  We used a recursive procedure $R(S, p)$, that assigns probabilities to the hands in the set $S$ that sum to probability $p$, according to the following procedure.
\begin{enumerate}
\item If $|S| =1$, then $\Pr(s) = p$.
\item Otherwise,
\begin{enumerate}
\item Choose $p_1$ uniformly at random from the interval $(0, p)$, and let $p_2 = p-p_1$.
\item Let $S_1 \subset S$ and $S_2 = S \setminus S_1$ such that $|S_1| = \left\lfloor|S|/2\right\rfloor$ and all of the hands in $S_1$ have a hand strength no greater than hands in $S_2$.  Hand strength is the probability of a hand beating a uniformly selected random hand from the current public state.
\item Use $R(S_1, p_1)$ and $R(S_2, p_2)$ to assign probabilities to hands in $S=S_1 \bigcup S_2$.
\end{enumerate}
\end{enumerate}
Generating a range involves invoking $R(\mbox{\em all hands}, 1)$.
To obtain the target counterfactual values for the generated poker situations for the main networks, the situations were approximately solved using 1,000 iterations of CFR$^+$ with only betting actions fold, call, a pot-sized bet, and all-in.  For the turn network, ten million poker turn situations (from after the turn card is dealt) were generated and solved with 6,144 CPU cores of the Calcul Qu\'{e}bec MP2 research cluster, using over 175 core years of computation time.  For the flop network, one million poker flop situations (from after the flop cards are dealt) were generated and solved.  These situations were solved using DeepStack's depth limited solver with the turn network used for the counterfactual values at public states immediately after the turn card.  We used a cluster of 20 GPUS and one-half of a GPU year of computation time.  For the auxiliary network, ten million situations were generated and the target values were obtained by enumerating all 22,100 possible flops and averaging the counterfactual values from the flop network's output.  

\paragraph*{Neural Network Training.}
All networks were trained using built-in Torch7 libraries, with the Adam stochastic gradient descent procedure \cite{kingma2014adam} minimizing the average of the Huber losses~\cite{huber1964} over the counterfactual value errors.  Training used a mini-batch size of 1,000, and a learning rate 0.001, which was decreased to 0.0001 after the first 200 epochs. 
Networks were trained for approximately 350 epochs over two days on a single GPU, and the epoch with the lowest validation loss was chosen. 

\paragraph*{Neural Network Range Representation.}
In order to improve generalization over input player ranges, we map the distribution of individual hands (combinations of public and private cards) into distributions of buckets.  The buckets were generated using a clustering-based abstraction technique, which cluster strategically similar hands using $k$-means clustering with earth mover's distance over hand-strength-like features~\cite{Johanson13:Abstraction,Ganzfried14:EMD}.  For both the turn and flop networks we used 1,000 clusters and map the original ranges into distributions over these clusters as the first layer of the neural network (see Figure~3 of the main article).  This bucketing step was not used on the auxiliary network as there are only 169 strategically distinct hands pre-flop, making it feasible to input the distribution over distinct hands directly.

\paragraph*{Neural Network Accuracies.}
The turn network achieved an average Huber loss of 0.016 of the pot size on the training set and 0.026 of the pot size on the validation set.   The flop network, with a much smaller training set, achieved an average Huber loss of 0.008 of the pot size on the training set, but 0.034 of the pot size on the validation set.  Finally, 
 the auxiliary network had average Huber losses of 0.000053 and 0.000055 on the training and validation set, respectively.
Note that there are, in fact, multiple Nash equilibrium solutions to these poker situations, with each giving rise to different counterfactual value vectors.  So, these losses may overestimate the true loss as the network may accurately model a different equilibrium strategy.

\paragraph*{Number of Hidden Layers.}
We observed in early experiments that the neural network had a lower validation loss with an increasing number of hidden layers.  From these experiments, we chose to use seven hidden layers in an attempt to tradeoff accuracy, speed of execution, and the available memory on the GPU.  The result of a more thorough experiment examining the turn network accuracy as a function of the number of hidden layers is in Figure~\ref{fig:nn-errors-by-layers}.  It appears that seven hidden layers is more than strictly necessary as the validation error does not improve much beyond five.  However, all of these architectures were trained using the same ten million turn situations.  With more training data it would not be surprising to see the larger networks see a further reduction in loss due to their richer representation power.

\begin{figure}[t]
\centering
\includegraphics[width=0.5\textwidth]{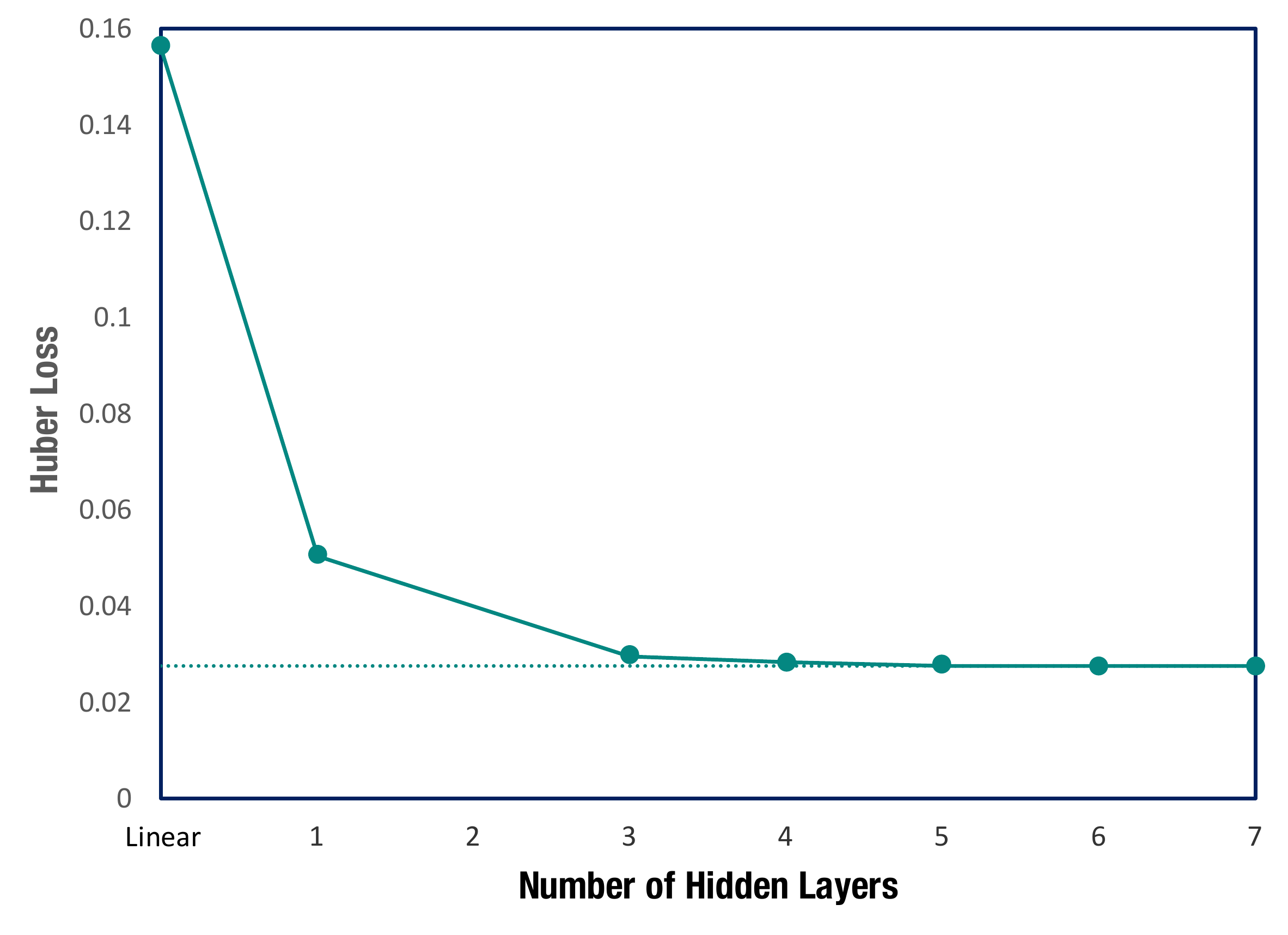}
\caption{Huber loss with different numbers of hidden layers in the neural network.}\label{fig:nn-errors-by-layers}
\end{figure}

\section*{Proof of Theorem 1}

The formal proof of Theorem 1, which establishes the soundness of DeepStack's
depth-limited continual re-solving, is conceptually easy to follow.   It
requires three parts.  First, we establish that the
exploitability introduced in a re-solving step has two linear components; one due
to approximately solving the subgame, and one due to the error in DeepStack's
counterfactual value network (see Lemmas 1 through 5).  Second, we enable estimates of
subgame counterfactual values that do not arise from actual subgame strategies (see Lemma 6).
Together, parts one and two enable us to use DeepStack's counterfactual value
network for a single re-solve.\footnote{The first part is a generalization and improvement
on the re-solving exploitability bound given by Theorem 3 in Burch et al.~\cite{cprg:cfrd},
and the second part generalizes the bound on decomposition regret given by Theorem 2 of
the same work.}  Finally, we show that using the opponent's values from the
best action, rather than the observed action, does not increase overall
exploitability (see Lemma 7).  This allows us to carry forward estimates of the opponent's
counterfactual value, enabling continual re-solving.  Put together,
these three parts bound the error after any finite number of continual re-solving
steps, concluding the proof.  We now formalize each step.

There are a number of concepts we use throughout this section.
We use the notation from Burch et al.~\cite{cprg:cfrd} without any
introduction here.  We assume player 1 is performing the continual
re-solving.  We call player 2 the opponent.
We only consider the re-solve player's strategy $\sigma$, as the opponent is always using a best response to $\sigma$.
All values are considered with respect to the opponent, unless specifically stated.
We say $\sigma$ is $\epsilon$-exploitable if the opponent's best
response value against $\sigma$ is no more than $\epsilon$ away from
the game value for the opponent.

A public state $S$ corresponds to the root of an imperfect information subgame.
We write $\mathcal{I}_2^S$ for the collection of player 2 information sets in $S$. 
Let $G\langle S, \sigma, w\rangle$ be the subtree gadget game (the re-solving game of Burch et al.~\cite{cprg:cfrd}), where $S$ is some public state, $\sigma$ is used to get player 1 reach probabilities $\pi^{\sigma}_{-2}(h)$ for each $h \in S$, and $w$ is a vector where $w_I$ gives the value of player 2 taking the terminate action (T) from information set $I \in \mathcal{I}_2^S$.
Let 
\begin{align*}
\text{BV}_I(\sigma) = \max_{\sigma^*_2}\sum_{h \in I} \pi^{\sigma}_{-2}(h) u^{\sigma,\sigma_2^*}(h) / \pi^{\sigma}_{-2}(I),
\end{align*}
be the counterfactual value for $I$ given we play $\sigma$ and our opponent is playing a best response.
For a subtree strategy $\sigma^S$, we write $\sigma \to \sigma^S$ for the strategy that plays according to $\sigma^S$ for any state in the subtree and according to $\sigma$ otherwise.  For the gadget game $G\langle S, \sigma, w\rangle$, the gadget value of a subtree strategy $\sigma^S$ is defined to be:
\begin{align*}
  \text{GV}^S_{w,\sigma}(\sigma^S) = \sum_{I \in \mathcal{I}^S_2} \max(w_I,\text{BV}_I(\sigma\to\sigma^S)),
\end{align*}
and the underestimation error is defined to be:
\begin{align*}
  \text{U}^S_{w,\sigma} = \min_{\sigma^S} \text{GV}^S_{w,\sigma}(\sigma^S) - \sum_{I \in \mathcal{I}^S_2}w_I.
\end{align*}

\begin{lemma}
  The game value of a gadget game $G\langle S, \sigma, w\rangle$ is 
\begin{equation*}\sum_{I \in \mathcal{I}^S_2} w_I + \text{U}^S_{w,\sigma}.\end{equation*}
  \label{lem:underestimation}
\end{lemma}
\begin{proof}
Let $\tilde{\sigma}_2^S$ be a gadget game strategy for player $2$ which must choose from the F and T actions at starting information set $I$.
Let $\tilde{u}$ be the utility function for the gadget game.
\begin{align*}
\min_{\sigma_1^S}\max_{\tilde{\sigma}_2^S} \tilde{u}(\sigma_1^S,\tilde{\sigma}_2^S)
&= \min_{\sigma_1^S}\max_{\sigma_2^S}\sum_{I \in \mathcal{I}^S_2} \frac{\pi^{\sigma}_{-2}(I)}{\sum_{I' \in \mathcal{I}^S_2} \pi^{\sigma}_{-2}(I')} \max_{a \in \{F,T\}} \tilde{u}^{\sigma^S}(I,a)\\
&= \min_{\sigma_1^S}\max_{\sigma_2^S}\sum_{I \in \mathcal{I}^S_2} \max( w_I, \sum_{h \in I} \pi_{-2}^{\sigma}(h) u^{\sigma^S}(h))
\intertext{A best response can maximize utility at each information set independently:}
&= \min_{\sigma_1^S}\sum_{I \in \mathcal{I}^S_2} \max( w_I, \max_{\sigma_2^S} \sum_{h \in I} \pi_{-2}^{\sigma}(h) u^{\sigma^S}(h))\\
&= \min_{\sigma_1^S}\sum_{I \in \mathcal{I}^S_2} \max( w_I, \text{BV}_I(\sigma\to\sigma_1^S))\\
&= \text{U}^S_{w,\sigma} + \sum_{I \in \mathcal{I}^S_2}  w_I
\end{align*}
\end{proof}

\begin{lemma}
  If our strategy $\sigma^S$ is $\epsilon$-exploitable in the gadget game
  $G\langle S,\sigma,w\rangle$, then $\text{GV}^S_{w,\sigma}(\sigma^S) \le \sum_{I \in \mathcal{I}^S_2} w_I +
  \text{U}^S_{w,\sigma} + \epsilon$
  \label{lem:resolve_underestimation}
\end{lemma}
\begin{proof}
  This follows from Lemma~\ref{lem:underestimation} and the definitions
  of $\epsilon$-Nash, $\text{U}^S_{w,\sigma}$, and $\text{GV}^S_{w,\sigma}(\sigma^S)$.
\end{proof}

\begin{lemma}
  Given an $\epsilon_O$-exploitable $\sigma$ in the original game, if
  we replace a subgame with a strategy $\sigma^S$ such that
  $\text{BV}_I(\sigma \to \sigma^S) \le w_I$ for all $I \in \mathcal{I}^S_2$, then the new combined
  strategy has an exploitability no more than $\epsilon_O + \text{EXP}^S_{w,\sigma}$ where
  \begin{align*}
    \text{EXP}^S_{w,\sigma} = \sum_{I \in \mathcal{I}^S_2} \max( \text{BV}_I( \sigma ), w_I ) - \sum_{I \in \mathcal{I}^S_2} \text{BV}_I( \sigma )
  \end{align*}
  \label{lem:error}
\end{lemma}
\begin{proof}
  We only care about the information sets where the opponent's
  counterfactual value increases, and a worst case upper bound occurs
  when the opponent best response would reach every such information
  set with probability 1, and never reach information sets where the
  value decreased.

Let $Z[S] \subseteq Z$ be the set of terminal states reachable from some $h \in S$ and let $v_2$ be the game value of the full game for player 2. Let $\sigma_2$ be a best response to $\sigma$ and let $\sigma^S_2$ be the part of $\sigma_2$ that plays in the subtree rooted at $S$. Then necessarily $\sigma^S_2$ achieves counterfactual value $\text{BV}_I(\sigma)$ at each $I \in \mathcal{I}^S_2$.

\begin{alignat*}{2}
\max_{\sigma_2^*} &\mathrlap{(u(\sigma\to\sigma^S,\sigma_2^*))}\\
&= \max_{\sigma_2^*} &&\biggl[\sum_{z \in Z[S]} \pi_{-2}^{\sigma \to \sigma^S}(z) \pi_2^{\sigma_2^*}(z) u(z) + \sum_{z \in Z \setminus Z[S]} \pi_{-2}^{\sigma \to \sigma^S}(z) \pi_2^{\sigma_2^*}(z) u(z) \biggr]\\
&= \max_{\sigma_2^*} &&\biggl[\sum_{z \in Z[S]} \pi_{-2}^{\sigma \to \sigma^S}(z) \pi_2^{\sigma_2^*}(z) u(z) - \sum_{z \in Z[S]} \pi_{-2}^{\sigma}(z) \pi_2^{\sigma_2^*\to\sigma_2^S}(z)u(z)\\
&&&+  \sum_{z \in Z[S]} \pi_{-2}^{\sigma}(z) \pi_2^{\sigma_2^*\to\sigma_2^S}(z)u(z)+ \sum_{z \in Z \setminus Z[S]} \pi_{-2}^{\sigma}(z) \pi_2^{\sigma_2^*}(z) u(z) \biggr]\\
&\leq \max_{\sigma_2^*} &&\biggl[\sum_{z \in Z[S]} \pi_{-2}^{\sigma \to \sigma^S}(z) \pi_2^{\sigma_2^*}(z) u(z) - \sum_{z \in Z[S]} \pi_{-2}^{\sigma}(z) \pi_2^{\sigma_2^*\to\sigma_2^S}(z)u(z)\biggr]\\
&&&+\max_{\sigma_2^*}\biggl[ \sum_{z \in Z[S]} \pi_{-2}^{\sigma}(z) \pi_2^{\sigma_2^*\to\sigma_2^S}(z)u(z)+ \sum_{z \in Z \setminus Z[S]} \pi_{-2}^{\sigma}(z) \pi_2^{\sigma_2^*}(z) u(z) \biggr]\\
&\leq \max_{\sigma_2^*} &&\biggl[\sum_{I \in \mathcal{I}^S_2} \sum_{h \in I}\pi_{-2}^\sigma(h)\pi_2^{\sigma_2^*}(h)u^{\sigma^S,\sigma_2^*}(h) \\
  &&& - \sum_{I \in \mathcal{I}^S_2} \sum_{h \in I}\pi_{-2}^\sigma(h)\pi_2^{\sigma_2^*}(h)u^{\sigma,\sigma_2^S}(h)\biggr] + \max_{\sigma_2^*} (u(\sigma,\sigma_2^*))
\intertext{By perfect recall $\pi_2(h)=\pi_2(I)$ for each $h \in I$:}
&\leq \max_{\sigma_2^*} &&\biggl[\sum_{I \in \mathcal{I}^S_2} \pi_2^{\sigma_2^*}(I)\biggl(\sum_{h \in I}\pi_{-2}^\sigma(h)u^{\sigma^S,\sigma_2^*}(h) - \sum_{h \in I}\pi_{-2}^\sigma(h)u^{\sigma,\sigma_2^S}(h)\biggr)\biggr]
\\ &&&+ v_2 + \epsilon_O\\
&= \max_{\sigma_2^*} &&\biggl[\mathrlap{\sum_{I \in \mathcal{I}^S_2}\pi_2^{\sigma_2^*}(I)\pi_{-2}^{\sigma}(I)\biggl(\text{BV}_I(\sigma \to \sigma^S)-\text{BV}_I(\sigma)\biggr)\biggr]+ v_2 + \epsilon_O}\\
&\leq &&\biggl[\sum_{I \in \mathcal{I}^S_2}\max(\text{BV}_I(\sigma \to \sigma^S)-\text{BV}_I(\sigma),0)\biggr]+ v_2 + \epsilon_O\\
&\leq &&\biggl[\sum_{I \in \mathcal{I}^S_2}\max(w_I-\text{BV}_I(\sigma),\text{BV}_I(\sigma)-\text{BV}_I(\sigma))\biggr]+ v_2 + \epsilon_O\\
&= &&\biggl[\sum_{I \in \mathcal{I}^S_2}\max(\text{BV}_I(\sigma),w_I) - \sum_{I \in \mathcal{I}^S_2} \text{BV}_I( \sigma )\biggr]+ v_2 + \epsilon_O
\end{alignat*}
\end{proof}

\begin{lemma}
  Given an $\epsilon_O$-exploitable $\sigma$ in the original game, if
  we replace the strategy in a subgame with a strategy $\sigma^S$ that
  is $\epsilon_S$-exploitable in the gadget game $G\langle S, \sigma,
  w\rangle$, then the new combined strategy has an exploitability no
  more than $\epsilon_O + \text{EXP}^S_{w,\sigma} +
  \text{U}^S_{w,\sigma} + \epsilon_S$.
  \label{lem:resolve_error}
\end{lemma}
\begin{proof}
  We use that $\max(a,b) = a + b - \min(a,b)$. From applying Lemma~\ref{lem:error} with\\ 
  ${w_I = \text{BV}_I(\sigma \to \sigma^S)}$ and expanding 
  ${\text{EXP}^S_{\text{BV}(\sigma \to \sigma^S),\sigma}}$ 
  we get exploitability no more than\\ ${\epsilon_O
  - \sum_{I \in \mathcal{I}^S_2} \text{BV}_I( \sigma )}$ plus
  \begin{align*}
    \sum_{I \in \mathcal{I}^S_2} &\max( \text{BV}_I( \sigma \to \sigma^S), \text{BV}_I( \sigma ) )  \\
    &\le \sum_{I \in \mathcal{I}^S_2} \max( \text{BV}_I( \sigma \to \sigma^S ), \max( w_I, \text{BV}_I( \sigma ) )  \\
    &= \sum_{I \in \mathcal{I}^S_2} \bigl( \text{BV}_I( \sigma \to \sigma^S ) + \max( w_I, \text{BV}_I( \sigma )) \\
    &\qquad - \min( \text{BV}_I( \sigma \to \sigma^S ), \max( w_I, \text{BV}_I( \sigma ))) \bigr)\\
    &\leq \sum_{I \in \mathcal{I}^S_2} \bigl( \text{BV}_I( \sigma \to \sigma^S ) + \max( w_I, \text{BV}_I( \sigma )) \\
    &\qquad - \min( \text{BV}_I( \sigma \to \sigma^S ), w_I) \bigr)\\
    &= \sum_{I \in \mathcal{I}^S_2} \bigl( \max( w_I, \text{BV}_I( \sigma )) + \max( w_I, \text{BV}_I( \sigma\to\sigma^S ) ) - w_I \bigr) \\
    &= \sum_{I \in \mathcal{I}^S_2} \max( w_I, \text{BV}_I( \sigma ) ) + \sum_{I \in \mathcal{I}^S_2} \max( w_I, \text{BV}_I( \sigma\to\sigma^S ))  - \sum_{I \in \mathcal{I}^S_2} w_I \\
  \end{align*}
  From Lemma~\ref{lem:resolve_underestimation} we get
  \begin{align*}
    \le \sum_{I \in \mathcal{I}^S_2} \max( w_I, \text{BV}_I( \sigma ) ) + \text{U}^S_{w,\sigma} + \epsilon_S \\
  \end{align*}
  Adding $\epsilon_O - \sum_I \text{BV}_I( \sigma )$ we get the
  upper bound $\epsilon_O + \text{EXP}^S_{w,\sigma} + \text{U}^S_{w,\sigma} + \epsilon_S$.
\end{proof}

\begin{lemma}
  Assume we are performing one step of re-solving on subtree $S$, with
  constraint values $w$ approximating opponent best-response values to
  the previous strategy $\sigma$, with an approximation error bound
  $\sum_I |w_I-\text{BV}_I(\sigma)| \le \epsilon_E$. Then we have
  $\text{EXP}^S_{w,\sigma} + \text{U}^S_{w,\sigma} \le \epsilon_E$.
  \label{lem:resolve_step}
\end{lemma}
\begin{proof}
  $\text{EXP}^S_{w,\sigma}$ measures the amount that the $w_I$ exceed $\text{BV}_I(\sigma)$, while $\text{U}^S_{w,\sigma}$ bounds the amount that the $w_I$ underestimate $\text{BV}_I(\sigma\to\sigma^S)$ for any $\sigma^S$, including the original $\sigma$. Thus, together they are bounded by $|w_I-\text{BV}_I(\sigma)|$:
  \begin{align*}
    \text{EXP}^S_{w,\sigma} + \text{U}^S_{w,\sigma} &= \sum_{I \in \mathcal{I}^S_2} \max( \text{BV}_I( \sigma ), w_I ) -
  \sum_{I \in \mathcal{I}^S_2} \text{BV}_I( \sigma )\\
    &\qquad+ \min_{\sigma^S} \sum_{I \in \mathcal{I}^S_2} \max(w_I,\text{BV}_I(\sigma\to\sigma^S)) - \sum_{I \in \mathcal{I}^S_2}w_I\\
    &\leq \sum_{I \in \mathcal{I}^S_2} \max( \text{BV}_I( \sigma ), w_I ) - \sum_{I \in \mathcal{I}^S_2} \text{BV}_I( \sigma )\\
    &\qquad + \sum_{I \in \mathcal{I}^S_2} \max(w_I,\text{BV}_I(\sigma)) - \sum_{I \in \mathcal{I}^S_2}w_I\\
    &= \sum_{I \in \mathcal{I}^S_2} \left[\max(w_I - \text{BV}_I(\sigma),0) + \max(\text{BV}_I(\sigma)-w_I,0)\right]\\
    &= \sum_{I \in \mathcal{I}^S_2} |w_I - \text{BV}_I(\sigma)| \leq \epsilon_E
  \end{align*}
\end{proof}

\begin{lemma}
  Assume we are solving a game $G$ with $T$ iterations of CFR-D where
  for both players $p$, subtrees $S$, and times $t$, we use subtree
  values $v_I$ for all information sets $I$ at the root of $S$ from
  some suboptimal black box estimator. If the estimation error is
  bounded, so that $\min_{\sigma_S^* \in \text{NE}_S} \sum_{I \in
    \mathcal{I}^S_2} |v^{\sigma_S^*} (I)-v_I| \le \epsilon_E$, then
  the trunk exploitability is bounded by $k_G/\sqrt{T} +
  j_G\epsilon_E$ for some game specific constant $k_G,j_G \ge 1$ which
  depend on how the game is split into a trunk and subgames.
  \label{lem:cfrd_estimator}
\end{lemma}
\begin{proof}
  This follows from a modified version the proof of Theorem 2 of
  Burch~et al.~\cite{cprg:cfrd}, which uses a fixed error $\epsilon$
  and argues by induction on information sets. Instead, we argue
  by induction on entire public states.

  For every public state $s$, let $N_s$ be the number of subgames
  reachable from $s$, including any subgame rooted at $s$. Let $Succ(s)$
  be the set of our public states which are reachable from $s$ without going through
  another of our public states on the way. Note that if $s$ is in the trunk, then every
  $s' \in Succ(s)$ is in the trunk or is the root of a subgame. Let $D_{TR}(s)$
  be the set of our trunk public states reachable from $s$, including $s$ if $s$ is in the trunk.
  We argue that for any public state $s$ where we act in the trunk or at the root of a subgame
  \begin{equation}
  \label{eqn:inductiveregret}
  \sum_{I \in s} R^{T,+}_{full}(I) \leq \sum_{s' \in D_{TR}(s)} \sum_{I \in s'} R^{T,+}(I) + TN_s\epsilon_E
  \end{equation}
  First note that if no subgame is reachable from $s$, then $N_s = 0$
  and the statement follows from Lemma 7 of \cite{ZinkevichEtAl07}.
  For public states from which a subgame is reachable, we argue by
  induction on $|D_{TR}(s)|$. 
   
  For the base case, if $|D_{TR}(s)| = 0$ then $s$ is the root of a subgame $S$,
  and by assumption there is a Nash Equilibrium subgame strategy $\sigma_S^*$
  that has regret no more than $\epsilon_E$. If we implicitly play $\sigma_S^*$
  on each iteration of CFR-D, we thus accrue $\sum_{I \in s} R^{T,+}_{full}(I) \leq T\epsilon_E$.

  For the inductive hypothesis, we assume that (\ref{eqn:inductiveregret}) holds for all $s$
  such that $|D_{TR}(s)| < k$.

  Consider a public state $s$ where $|D_{TR}(s)| = k$. By Lemma 5 of \cite{ZinkevichEtAl07}
  we have
  \begin{align*}
  \sum_{I \in s} R^{T,+}_{full}(I) &\leq \sum_{I \in s}\left[R^T(I) + \sum_{I' \in Succ(I)}R^{T,+}_{full}(I)\right]\\
  &= \sum_{I \in s} R^T(I) + \sum_{s' \in Succ(s)} \sum_{I' \in s'} R^{T,+}_{full}(I')
  \end{align*}

  For each $s' \in Succ(s)$, $D(s') \subset D(s)$ and $s \not\in D(s')$, so $|D(s')|<|D(s)|$
  and we can apply the inductive hypothesis to show
  \begin{align*}
  \sum_{I \in s} R^{T,+}_{full}(I) &\leq \sum_{I \in s} R^T(I) + \sum_{s' \in Succ(s)} \left[\sum_{s'' \in D(s')}\sum_{I \in s''} R^{T,+}(I) + TN_{s'}\epsilon_E\right]\\
  &\leq \sum_{s' \in D(s)}\sum_{I \in s'} R^{T,+}(I) + T\epsilon_E \sum_{s' \in Succ(s)} N_{s'}\\
  &= \sum_{s' \in D(s)}\sum_{I \in s'} R^{T,+}(I) + T\epsilon_E N_s
  \end{align*}

  This completes the inductive argument. By using regret matching in the trunk, we ensure
  $R^T(I) \leq \Delta\sqrt{AT}$, proving the lemma for $k_G = \Delta|\mathcal{I_{TR}}|\sqrt{A}$
  and $j_G = N_{root}$.
\end{proof}

\begin{lemma}
  Given our strategy $\sigma$, if the opponent is acting at the root
  of a public subtree $S$ from a set of actions $A$, with opponent
  best-response values $\text{BV}_{I \cdot a}(\sigma)$ after each
  action $a \in A$, then replacing our subtree strategy with any
  strategy that satisfies the opponent constraints $w_I = \max_{a \in
    A} \text{BV}_{I \cdot a}(\sigma)$ does not increase our
  exploitability.
  \label{lem:opp_constraint_values}
\end{lemma}
\begin{proof}
  If the opponent is playing a best response, every counterfactual
  value $w_I$ before the action must either satisfy $w_I =
  \text{BV}_I(\sigma) = \max_{a \in A} \text{BV}_{I \cdot a}(\sigma)$,
  or not reach state $s$ with private information $I$. If we replace
  our strategy in S with a strategy $\sigma'_S$ such that
  $\text{BV}_{I \cdot a}(\sigma'_S) \le \text{BV}_I(\sigma)$ we
  preserve the property that $\text{BV}_I(\sigma') =
  \text{BV}_I(\sigma)$.
\end{proof}

\begin{theorem}
  Assume we have some initial opponent constraint values $w$ from a
  solution generated using at least $T$ iterations of CFR-D, we use at
  least $T$ iterations of CFR-D to solve each re-solving game, and we
  use a subtree value estimator such that $\min_{\sigma_S^* \in
    \text{NE}_S} \sum_{I \in \mathcal{I}^S_2} |v^{\sigma_S^*} (I)-v_I|
  \le \epsilon_E$, then after $d$ re-solving steps the exploitability
  of the resulting strategy is no more than $(d+1)k/\sqrt{T} +
  (2d+1)j\epsilon_E$ for some constants $k,j$ specific to both the
  game and how it is split into subgames.
  \label{thm:deepstack}
\end{theorem}
\begin{proof}
  Continual re-solving begins by solving from the root of the entire game, which we label as
  subtree $S_0$. We use CFR-D with the value estimator in place of subgame solving 
  in order to generate an initial strategy $\sigma_0$ for playing in $S_0$. 
  By Lemma~\ref{lem:cfrd_estimator}, the exploitability
  of $\sigma_0$ is no more than $k_0/\sqrt{T} + j_0\epsilon_E$.
  
  For each step of continual re-solving $i = 1,...,d$, we are re-solving some
  subtree $S_i$. From the previous step of re-solving, we have approximate opponent 
  best-response counterfactual values $\widetilde{\mathrm{BV}}_I(\sigma_{i-1})$ for each
  $I \in \mathcal{I}^{S_{i-1}}_2$, which by the estimator bound satisfy 
  $|\sum_{I \in \mathcal{I}^{S_{i-1}}_2} BV_I(\sigma_{i-1}) - \widetilde{\mathrm{BV}}_I(\sigma_{i-1})| \leq \epsilon_E$.
  Updating these values at each public state between $S_{i-1}$ and $S_{i}$ 
  as described in the paper yields approximate values $\widetilde{\mathrm{BV}}_I(\sigma_{i-1})$ 
  for each $I \in \mathcal{I}^{S_i}_2$ which by Lemma~\ref{lem:opp_constraint_values}
  can be used as constraints $w_{I,i}$ in re-solving. Lemma~\ref{lem:resolve_step} with these constraints
  gives us the bound $\text{EXP}^{S_i}_{w_i,\sigma_{i-1}} +
  \text{U}^{S_i}_{w_i,\sigma_{i-1}} \le \epsilon_E$.
  Thus by Lemma~\ref{lem:resolve_error} and Lemma~\ref{lem:cfrd_estimator} we
  can say that the increase in exploitability from $\sigma_{i-1}$ to
  $\sigma_i$ is no more than $\epsilon_E + \epsilon_{S_i} \leq \epsilon_E + k_i/\sqrt{T} + j_i\epsilon_E
  \leq  k_i/\sqrt{T} + 2j_i\epsilon_E$.

  Let $k = \max_i k_i$ and $j = \max_i j_i$. Then after $d$ re-solving
  steps, the exploitability is bounded by $(d+1)k/\sqrt{T} +
  (2d+1)j\epsilon_E$.
\end{proof}

\subsection*{Best-response Values Versus Self-play Values}
DeepStack uses self-play values within the continual re-solving
computation, rather than the best-response values described in
Theorem~\ref{thm:deepstack}. Preliminary tests using CFR-D to solve
smaller games suggested that strategies generated using self-play
values were generally less exploitable and had better one-on-one
performance against test agents, compared to strategies generated
using best-response values.  Figure~\ref{fig:selfplay-converge} shows
an example of DeepStack's exploitability in a particular river subgame 
with different numbers of re-solving iterations.  Despite lacking
a theoretical justification for its soundness, using self-play values appears
to converge to low exploitability strategies just as with using best-response values.

\begin{figure}[t]
\centering
\includegraphics[width=0.5\textwidth]{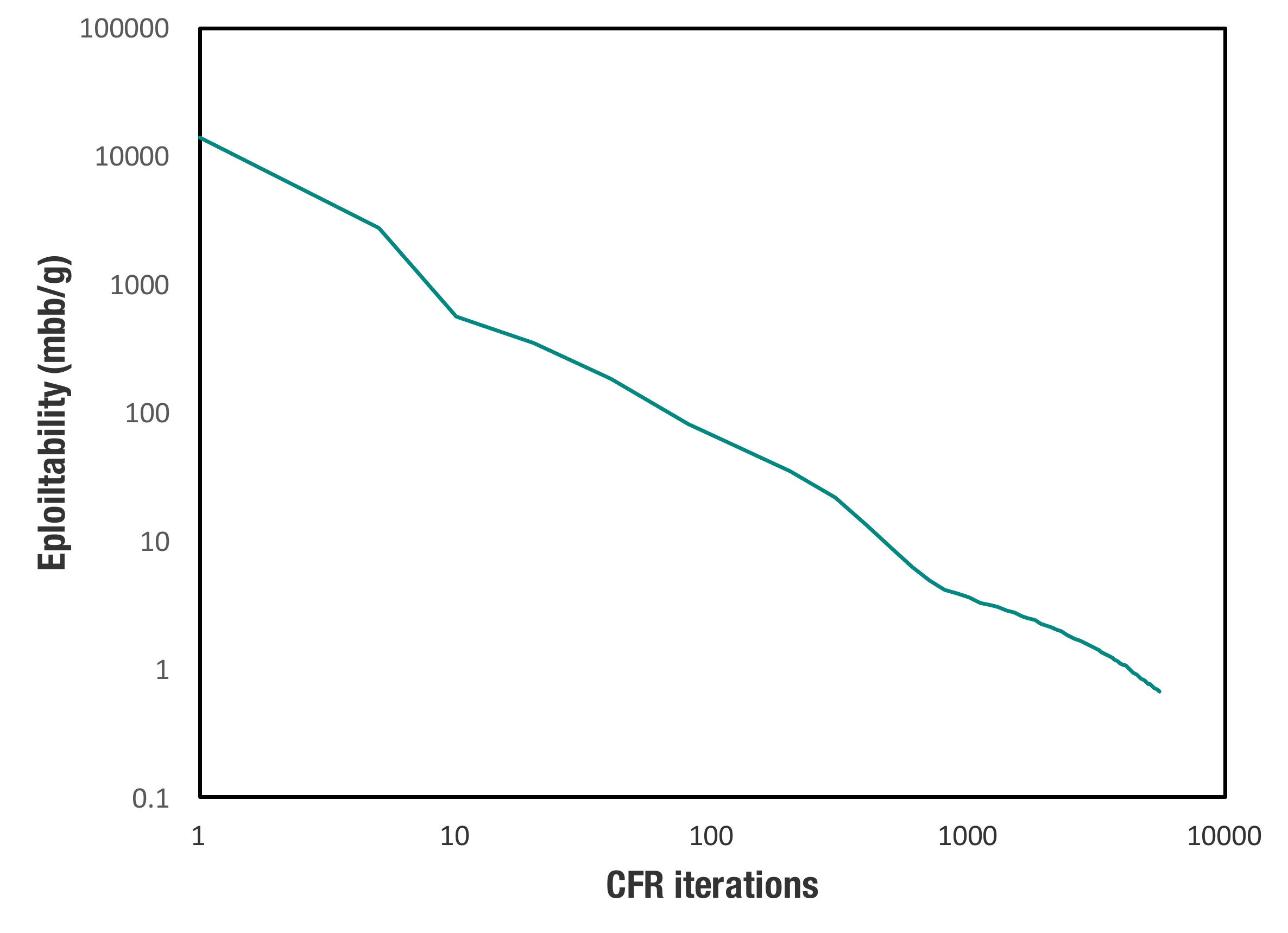}
\caption{DeepStack's exploitability within a particular public state at the start of the river as a function of the number of re-solving iterations.}\label{fig:selfplay-converge}
\end{figure}

One possible explanation for why self-play values work well with
continual re-solving is that at every re-solving step, we give away a
little more value to our best-response opponent because we are not
solving the subtrees exactly. If we use the self-play values for the
opponent, the opponent's strategy is slightly worse than a best
response, making the opponent values smaller and
counteracting the inflationary effect of an inexact solution.  
While this optimism could hurt us by setting unachievable
goals for the next re-solving step (an increased
$\text{U}^S_{w,\sigma}$ term), in poker-like games we find that the
more positive expectation is generally correct (a decreased
$\text{EXP}^S_{w,\sigma}$ term.)

\section*{Pseudocode}
Complete pseudocode for DeepStack's depth-limited continual re-resolving algorithm is in Algorithm~\ref{alg:pseudocode}.  Conceptually, DeepStack can be decomposed into four functions: \textproc{Re-solve}, \textproc{Values}, \textproc{UpdateSubtreeStrategies}, and \textproc{RangeGadget}. The main function is \textproc{Re-solve}, which is called every time DeepStack needs to take an action. It iteratively calls each of the other functions to refine the lookahead tree solution. After $T$ iterations, an action is sampled from the approximate equilibrium strategy at the root of the subtree to be played. 
According to this action, DeepStack's range, $\vec{r}_1$, and its opponent's counterfactual values, $\vec{v}_2$, are updated in preparation for its next decision point.

\let\veccopy\vec
\let\vec\mathbf

\begin{algorithm}[!htbp]
\small 
\caption{Depth-limited continual re-solving}
\label{alg:pseudocode}
\textbf{INPUT:} Public state $S$, player range $\vec{r}_1$ over our information sets in $S$, opponent counterfactual values $\vec{v}_2$ over their information sets in $S$, and player information set $I \in S$\\
\textbf{OUTPUT:} Chosen action $a$, and updated representation after the action $(S(a), \vec{r}_1(a), \vec{v}_2(a))$
\begin{algorithmic}[1]

\Statex
\Function{Re-solve}{$S,\vec{r}_1,\vec{v}_2,I$}
    \State $\sigma^0 \gets \text{arbitrary initial strategy profile}$
    \State $\vec{r}^0_2 \gets \text{arbitrary initial opponent range}$
    \State $R_G^0,R^0 \gets \vec{0}$ \Comment{Initial regrets for gadget game and subtree}
    \For{$t=1$ to $T$}
        \State $\vec{v}_1^{t},\vec{v}_2^{t} \gets$ \Call{Values}{$S,\sigma^{t-1},\vec{r}_1,\vec{r}^{t-1}_2,0$}
        \State $\sigma^t,R^t \gets$ \Call{UpdateSubtreeStrategies}{$S,\vec{v}_1^{t},\vec{v}_2^{t},R^{t-1}$}
        \State $\vec{r}^t_2,R_G^t \gets$    
            \Call{RangeGadget}{$\vec{v}_2,\vec{v}_2^{t}(S),R_G^{t-1}$}
    \EndFor
    \State $\overline{\sigma}^T \gets \frac{1}{T}\sum_{t=1}^T \sigma^t$ \Comment{Average the strategies}
    \State $a \sim \overline{\sigma}^T(\cdot |I)$ \Comment{Sample an action}
    \State $\vec{r}_1(a) \gets \langle \vec{r}_1, \sigma(a|\cdot) \rangle$ \Comment{Update the range based on the chosen action}
    \State $\vec{r}_1(a) \gets \vec{r}_1(a) / ||\vec{r}_1(a)||_1$ \Comment{Normalize the range}
    \State $\vec{v}_2(a) \gets \frac{1}{T}\sum_{t=1}^T \vec{v}^t_2(a)$ \Comment{Average of counterfactual values after action $a$}
    \State \Return $a,S(a),\vec{r}_1(a), \vec{v}_2(a)$
\EndFunction
\Statex

\Function{Values}{$S,\sigma,\vec{r}_1,\vec{r}_2,d$} \Comment{\parbox[t]{.6\linewidth}{Gives the counterfactual values of the subtree $S$ under $\sigma$,
    computed with a depth-limited lookahead.}}
    \If{$S$ is terminal}
        \State $\vec{v}_1(S) \gets U_S\vec{r}_2$ \Comment{Where $U_S$ is the matrix of the bilinear utility function at $S$, \hspace{2pt}~}
        \State $\vec{v}_2(S) \gets \vec{r}_1^\intercal U_S$ \hspace*{\fill}$U(S) = \vec{r}_1^\intercal U_S \vec{r_2}$, thus giving vectors of counterfactual values
        \State \Return $\vec{v}_1(S),\vec{v}_2(S)$
    \ElsIf{$d = \textrm{MAX-DEPTH}$}
        \State \Return \Call{NeuralNetEvaluate}{$S,\vec{r}_1,\vec{r}_2$}
    \EndIf
    \State $\vec{v}_1(S), \vec{v}_2(S) \gets \vec{0}$
    \For{action $a \in S$}
        \State $\vec{r}_{\text{Player}(S)}(a) \gets \langle \vec{r}_{\text{Player}(S)}, \sigma(a|\cdot) \rangle$
            \Comment Update range of acting player based on strategy
        \State $\vec{r}_{\text{Opponent}(S)}(a) \gets \vec{r}_{\text{Opponent}(S)}$
        \State $\vec{v}_1(S(a)), \vec{v}_2(S(a)) \gets \Call{Value}{S(a),\sigma,\vec{r}_1(a),\vec{r}_2(a),d+1}$
        \State $\vec{v}_{\text{Player}(S)}(S) \gets \vec{v}_{\text{Player}(S)}(S) + \sigma(a|\cdot)\vec{v}_{\text{Player(S)}}(S(a))$
            \Comment{Weighted average}
        \State $\vec{v}_{\text{Opponent}(S)}(S) \gets \vec{v}_{\text{Player}(S)}(S) + \vec{v}_{\text{Opponent(S)}}(S(a))$
            \Statex\Comment{\parbox[t]{.5\linewidth}{Unweighted sum, as our strategy is already included in opponent counterfactual values}}
    \EndFor
    \State \Return $\vec{v}_1, \vec{v}_2$
\EndFunction
\algstore{pseudocode}
\end{algorithmic}
\end{algorithm}

\begin{algorithm}
\small
\begin{algorithmic}
\algrestore{pseudocode}

\Function{UpdateSubtreeStrategies}{$S,\vec{v}_1,\vec{v}_2,R^{t-1}$}
    \For{$S' \in \{S\}\cup\text{SubtreeDescendants}(S)$ with $\text{Depth}(S') < \textrm{MAX-DEPTH}$}
        \For{action $a \in S'$}
            \State $R^t(a|\cdot) \gets R^{t-1}(a|\cdot) + \vec{v}_{\text{Player}(S')}(S'(a)) - \vec{v}_{\text{Player}(S')}(S')$
                \Statex\Comment{Update acting player's regrets}
        \EndFor
        \For{information set $I \in S'$}
            \State $\sigma^t(\cdot|I) \gets \frac{R^t(\cdot|I)^+}{\sum_{a}R^t(a|I)^+}$ \Comment{Update strategy with regret matching}
        \EndFor
    \EndFor
    \State \Return $\sigma^t, R^t$
\EndFunction
\Statex

\Function{RangeGadget}{$\vec{v}_2,\vec{v}_2^{t},R_G^{t-1}$} \Comment{\parbox[t]{.5\linewidth}{Let opponent choose to play in the subtree or receive the input value with each hand (see Burch~et~al.~\cite{cprg:cfrd})}}
    \State $\sigma_G(\textrm{F}|\cdot) \gets
        \frac{R_G^{t-1}(\textrm{F}|\cdot)^{+}}{R_G^{t-1}(\textrm{F}|\cdot)^{+} +
        R_G^{t-1}(\textrm{T}|\cdot)^{+}}$ \Comment{F is Follow action, T is Terminate}
    \State $\vec{r}_2^t \gets \sigma_G(\textrm{F}|\cdot)$
    \State $\vec{v}_G^{t} \gets \sigma_G(\textrm{F}|\cdot)\vec{v}_2^{t-1} +    
        (1-\sigma_G(\textrm{F}|\cdot))\vec{v}_2$ \Comment{Expected value of gadget strategy}
    \State $R_G^{t}(\textrm{T}|\cdot) \gets R_G^{t-1}(\textrm{T}|\cdot) + \vec{v}_2 - v_G^{t-1}$ \Comment{Update regrets}
    \State $R_G^{t}(\textrm{F}|\cdot) \gets R_G^{t-1}(\textrm{F}|\cdot) + \vec{v}_2^{t} - v_G^{t}$
    \State \Return $\vec{r}_2^t, R_G^t$
\EndFunction
\end{algorithmic}
\end{algorithm}

\let\vec\veccopy

\fi

\end{document}